\documentclass{article}

\usepackage{microtype}
\usepackage{graphicx}
\usepackage{subfigure}
\usepackage{booktabs} 

\usepackage{hyperref}

\usepackage[accepted]{icml2018_arxiv}

\usepackage{url}            
\usepackage{multirow}
\usepackage{amsmath}
\usepackage{amssymb}
\usepackage{amsthm}
\usepackage{here}
\usepackage{wrapfig}
\usepackage{centernot}
\usepackage{enumitem}

\newtheorem{theorem}{Theorem}
\newtheorem{lemma}{Lemma}
\newtheorem{proposition}{Proposition}
\newtheorem{definition}{Definition}
\newtheorem{corollary}{Corollary}
\newtheorem{assumption}{Assumption}

\newcommand{\defeq}{\overset{def}{=}}

\newcommand{\argmax}{{\rm arg}\max}
\newcommand{\argmin}{{\rm arg}\min}

\def\pd<#1>{\left\langle #1 \right\rangle}
\def\floor[#1]{\left\lfloor #1 \right\rfloor}
\def\ceil[#1]{\left\lceil #1 \right\rceil}
\newcommand{\dotp}[2]{#1^{\top}#2}

\newcommand{\realsp}{\mathbb{R}}
\newcommand{\naturalsp}{\mathbb{N}}
\newcommand{\expec}{\mathbb{E}}
\newcommand{\prob}{\mathbb{P}}

\newcommand{\fdim}{d}
\newcommand{\edim}{D}
\newcommand{\nclass}{c}
\newcommand{\ndata}{n}
\newcommand{\featuresp}{\mathcal{X}}
\newcommand{\labelsp}{\mathcal{Y}}
\newcommand{\tpr}{\nu}

\newcommand{\predict}[2]{\dotp{#1}{#2}}

\newcommand{\radcomp}{\Re}
\newcommand{\empradcomp}{\hat{\radcomp}}

\def\regrisk{\mathcal{R}}
\def\risk{\mathcal{L}}
\def\dumrisk{\mathcal{R}'}

\icmltitlerunning{Functional Gradient Boosting based on Residual Network Perception}

\begin{document}

\twocolumn[
\icmltitle{Functional Gradient Boosting based on Residual Network Perception}



\icmlsetsymbol{equal}{*}

\begin{icmlauthorlist}
\icmlauthor{Atsushi Nitanda}{tokyo,riken}
\icmlauthor{Taiji Suzuki}{tokyo,riken}
\end{icmlauthorlist}

\icmlaffiliation{tokyo}{Graduate School of Information Science and Technology, The University of Tokyo}
\icmlaffiliation{riken}{Center for Advanced Intelligence Project, RIKEN}

\icmlcorrespondingauthor{Atsushi Nitanda}{atsushi\_nitanda@mist.i.u-tokyo.ac.jp}
\icmlcorrespondingauthor{Taiji Suzuki}{taiji@mist.i.u-tokyo.ac.jp}

\icmlkeywords{Machine Learning, Classification, Deep Learning, Residual Network, Functional Gradient}

\vskip 0.3in
]



\printAffiliationsAndNotice{}  

\begin{abstract}
  Residual Networks (ResNets) have become state-of-the-art models in deep learning and 
  several theoretical studies have been devoted to understanding why ResNet works so well.
  One attractive viewpoint on ResNet is that it is optimizing the risk in a functional space by combining an ensemble of effective features. 
  In this paper, we adopt this viewpoint to construct a new {\it gradient boosting method}, which is known to be very powerful in data analysis.
  To do so, we formalize the gradient boosting perspective of ResNet mathematically using the notion of functional gradients
  and propose a new method called {\it ResFGB} for classification tasks by leveraging ResNet perception.
  Two types of generalization guarantees are provided from the optimization perspective:
  one is the margin bound and the other is the expected risk bound by the sample-splitting technique.
  Experimental results show superior performance of the proposed method over state-of-the-art methods such as LightGBM.
\end{abstract}
\sloppy
\section{Introduction}
Deep neural networks have achieved great success in classification tasks; in particular,
residual network (ResNet) \cite{he2016deep} and its variants such as wide-ResNet \cite{Zagoruyko2016WRN}, ResNeXt \cite{xie2017aggregated}, and DenseNet \cite{huang2017densely}
have become the most prominent architectures in computer vision.
Thus, to reveal a factor in their success, several studies have explored the behavior of ResNets and some promising perceptions have been advocated.
Concerning the behavior of ResNets, there are mainly two types of thoughts.
One is the ensemble views, which were pointed out in \citet{veit2016residual,littwin2016loss}.
They presented that ResNets are ensemble of shallower models using an unraveled view of ResNets.
Moreover, \citet{veit2016residual} enhanced their claim by showing that dropping or shuffling residual blocks does not affect the performance of ResNets experimentally.
The other is the optimization or ordinary differential equation views.
In \citet{jastrzebski2017residual}, it was observed experimentally that
ResNet layers iteratively move data representations along the negative gradient of the loss function with respect to hidden representations. 
Moreover, several studies \cite{weinan2017proposal,haber2017learning,chang2017reversible,chang2017multi,lu2017beyond} have pointed out that ResNet layers can be
regarded as discretization steps of ordinary differential equations.
Since optimization methods are constructed based on the discretization of gradient flows, these studies are closely related to each other.

On the other hand, gradient boosting \cite{mason1999boosting,friedman2001greedy} is known to be a state-of-the-art method in data analysis;
in particular, XGBoost \cite{chen2016xgboost} and LightGBM \cite{ke2017lightgbm} are notable because of their superior performance.
Although ResNets and gradient boosting are prominent methods in different domains, we notice an interesting similarity by recalling that
gradient boosting is an ensemble method based on iterative refinement by functional gradients for optimizing predictors.
However, there is a key difference between ResNets and gradient boosting methods.
While gradient boosting directly updates the predictor, ResNets iteratively optimize the feature extraction by stacking ResNet layers rather than the predictor,
according to the existing work. 

In this paper, leveraging this observation, we propose a new gradient boosting method called {\it ResFGB} for classification tasks based on ResNet perception,
that is, the feature extraction gradually grows by functional gradient methods in the space of feature extractions and
the resulting predictor naturally forms a ResNet-type architecture.
The expected benefit of the proposed method over usual gradient boosting methods is
that functional gradients with respect to feature extraction can learn a deep model rather than a shallow model like usual gradient boosting.
As a result, more efficient optimization is expected. 

In the theoretical analysis of the proposed method, we first formalize the gradient boosting perspective of ResNet mathematically using the notion of functional gradients in the space of feature extractions.
That is, we explain that optimization in that space is achieved by stacking ResNet layers.
We next show a good consistency property of the functional gradient, which motivates us to find feature extraction with small functional gradient norms for estimating the correct label of data.
This fact is very helpful from the optimization perspective because minimizing the gradient norm is much easier than minimizing the objective function without strong convexity.
Moreover, we show the margin maximization property of the proposed method and derive the margin bound
by utilizing this formalization and the standard complexity analysis techniques developed in \citet{koltchinskii2002empirical, bartlett2002rademacher},
which guarantee the generalization ability of the method.
This bound gives theoretical justification for minimizing functional gradient norms in terms of both optimization and better generalization.
Namely, we show that faster convergence of functional gradient norms leads to smaller classification errors.
As for another generalization guarantee, we also provide convergence analysis of the sample-splitting variant of the method for the expected risk minimization.
We finally show superior performance, empirically, of the proposed method over state-of-the-art methods including LightGBM.

\paragraph{Related work.}
Several studies have attempted to grow neural networks sequentially based on the boosting theory.
\citet{bengio2006convex} introduced convex neural networks consisting of a single hidden layer, and proposed a gradient boosting-based method
in which linear classifiers are incrementally added with their weights. However, the expressive power of the convex neural network is somewhat limited because 
the method cannot learn deep architectures.
\citet{moghimi2016boosted} proposed boosted convolutional neural networks and showed superior empirical performance on fine-grained classification tasks,
where convolutional neural networks are iteratively added, 
while our method constructs a deeper network by iteratively adding layers.
\citet{cortes2017adanet} proposed AdaNet to adaptively learn both the structure of the network and its weight, and 
provided data-dependent generalization guarantees for an adaptively learned network; however, the learning strategy quite differs from our method and the convergence rate is unclear.
The most related work is BoostResNet \cite{huang2017learning}, which constructs ResNet iteratively like our method; however, this method is based on an different theory
rather than functional gradient boosting with a constant learning rate.
This distinction makes the different optimization-generalization tradeoff.
Indeed, our method exhibits a tradeoff with respect to the learning rate, which recalls perception of usual functional gradient boosting methods, namely
a smaller learning rate leads to a good generalization performance.

\section{Preliminary}
In this section, we provide several notations and describe a problem setting of the classification.
An important notion in this paper is the functional gradient, which is also introduced in this section.

\paragraph{Problem setting.}
Let $\featuresp = \realsp^\fdim$ and $\labelsp$ be a feature space and a finite label set of cardinal $\nclass$, respectively.
We denote by $\tpr$ a true Borel probability measure on $\featuresp \times \labelsp$ and
by $\tpr_{\ndata}$ an empirical probability measure of samples $(x_i,y_i)_{i=1}^\ndata$ independently drawn from $\tpr$,
i.e., $d\tpr_{\ndata}(X,Y)=\sum_{i=1}^\ndata \delta_{(x_i,y_i)}(X,Y)dXdY/\ndata$, where $\delta$ denotes the Dirac delta function.
We denote by $\tpr_X$ the marginal distribution on $X$ and by $\tpr(\cdot| X)$ the conditional distribution on $Y$.
We also denote empirical variants of these distributions by $\tpr_{\ndata,X}$ and $\tpr_{\ndata}(\cdot | X)$.
In general, for a probability measure $\mu$, we denote by $\expec_{\mu}$ the expectation with respect to a random variable 
according to $\mu$, by $L_2(\mu)$ the space of square-integrable functions with respect to $\mu$,
and by $L_2^q(\mu)$ the product space of $L_2(\mu)$ equipped with $\pd<\cdot,\cdot>_{L_2^q(\mu)}$-inner product: for $\forall \xi,\forall \zeta \in L_2^q(\mu)$,
\vspace{-2mm}
\begin{equation*}
  \pd<\xi,\zeta>_{L_2^q(\mu)}
  \defeq \expec_{\mu}[\dotp{\xi(X)}{\zeta(X)}] 
  = \expec_{\mu}\left[ \sum_{j=1}^q \xi_j(X)\zeta_j(X) \right].
\end{equation*}
We also use the following norm: for $\forall p \in (0,2]$ and $\forall \xi \in L_2^q(\mu)$,
$\|\xi\|_{L_p^q(\mu)}^p \defeq \expec_{\mu}[ \|\xi(X)\|_2^p]
= \expec_{\mu}\left[ (\sum_{j=1}^q \xi_j^2(X))^{p/2} \right]$.

The ultimate goal in classification problems is to find a predictor $f \in L_2^\nclass(\tpr_X)$ such that $\argmax_{y\in \labelsp} f_y(x)$ correctly classifies its label.
The quality of the predictor is measured by a loss function $l(\zeta,y) \geq 0$. 
A typical choice of $l$ in multiclass classification problems is $l(\zeta,y) = - \log( \exp(\zeta_y)/\sum_{\overline{y}\in \labelsp}\exp(\zeta_{\overline{y}}) )$,
which is used for the multiclass logistic regression.
The goal of classification is achieved by solving the expected risk minimization problem:
\begin{equation}
  \min_{f \in L_2^\nclass(\tpr_X)} \left\{ \risk(f) \defeq
  \expec_{\tpr}[l(f(X),Y)] \right\}. \label{original_expected_risk_prob}
\end{equation}
However, the true probability measure $\tpr$ is unknown, so we approximate $\risk$ using the observed data probability measure $\tpr_{\ndata}$ and solve the empirical risk minimization problems:
\begin{equation}
  \min_{f \in L_2^\nclass(\tpr_X)} \left\{ \risk_{\ndata}(f) \defeq
  \expec_{\tpr_{\ndata}}[l(f(X),Y)] \right\}. \label{original_empirical_risk_prob}
\end{equation}
In general, some regularization is needed for the problem (\ref{original_empirical_risk_prob}) to guarantee generalization.
In this paper, we rely on early stopping \cite{zhang2005boosting} and some restriction on optimization methods for solving the problem.

Similar to neural networks, we split the predictor $f$ into the feature extraction and linear predictor, that is, $f(x) = \predict{w}{\phi(x)}$,
where $w \in \realsp^{\fdim\times \nclass}$ is a weight for the last layer and $\phi \in L_2^{\fdim}(\tpr_X)$ is a feature extraction from $\featuresp$ to $\featuresp$.
For simplicity, we also denote $l(z,y,w)=l(\predict{w}{z},y)$.
Usually, $\phi$ is parameterized by a neural network and optimized using the stochastic gradient method.
In this paper, we propose a way to optimize $\phi$ in $L_2^\fdim(\tpr_X)$ via the following problem:

\begin{equation}
  \hspace{-3mm}  \min_{ \substack{w \in \realsp^{\fdim \times \nclass} \\ \phi \in L_2^\fdim(\tpr_X) }}
  \left\{ \regrisk(\phi,w) \defeq
  \expec_{\tpr}[l(\phi(X),Y,w)] + \frac{\lambda}{2}\|w\|_2^2 \right\} \label{expected_risk_prob}
\end{equation}
where $\lambda>0$ is a regularization parameter to stabilize the optimization procedure and $\|\cdot\|_{2}$ for $w$ is a Euclidean norm.
When we focus on the problem with respect to $\phi$, we use the notation $\regrisk(\phi)\defeq \min_{w \in \realsp^{\fdim \times \nclass}} \regrisk(\phi,w)$.
We also denote by $\regrisk_\ndata(\phi,w)$ and $\regrisk_\ndata(\phi)$ empirical variants of $\regrisk(\phi,w)$ and $\regrisk(\phi)$, respectively,
which are defined by replacing $\expec_{\tpr}$ by $\expec_{\tpr_{\ndata}}$.
In this paper, we denote by $\partial$ the partial derivative and its subscript indicates the direction.

\paragraph{Functional gradient.}
The key notion used for solving the problem is the functional gradient in function spaces.
Since they are taken in some function spaces, we first introduce Fr\'{e}chet differential in general Hilbert spaces.

\begin{definition}
  Let $\mathcal{H}$ be a Hilbert space and $h$ be a function on $\mathcal{H}$.
  For $\xi \in \mathcal{H}$, we call that $h$ is Fr\'{e}chet differentiable at $\xi$ in $\mathcal{H}$
  when there exists an element $\nabla_\xi h(\xi) \in \mathcal{H}$ such that
  \begin{equation*}
    h(\zeta) = h(\xi) + \pd< \nabla_\xi h(\xi), \zeta-\xi>_{\mathcal{H}} + o(\|\xi-\zeta\|_{\mathcal{H}}).
  \end{equation*}
  Moreover, for simplicity, we call $\nabla_\xi h(\xi)$ Fr\'{e}chet differential or functional gradient.
\end{definition}

We here make an assumption to guarantee Fr\'{e}chet differentiability of $\regrisk, \regrisk_\ndata$,
which is valid for multiclass logistic loss: $l(z,y,w) = - \log( \exp(\predict{w_y}{z})/\sum_{\overline{y}\in \labelsp}\exp(\predict{w_{\overline{y}}}{z}) )$.
\begin{assumption} \label{smoothness_assumption_1}
  The loss function $l(\zeta,y): \realsp^{\nclass}\times \labelsp \rightarrow \realsp$ is a non-negative $\mathcal{C}^2$-convex function
  with respect to $\zeta$ and satisfies the following smoothness:
  There exists a positive real number $A$ such that $\|\partial_\zeta^2 l(\zeta,y)\| \leq A$ $(\forall (\zeta,y) \in \realsp^{\nclass}\times \labelsp)$,
  where $\|\cdot\|$ is the spectral norm.
\end{assumption}
Note that under this assumption, the following bound holds:
\[ \|\partial_z^2 l(z,y,w)\| \leq Ar^2\ for\ z \in \featuresp, y \in \labelsp, w \in B_r(0), \]
where $B_r(0)\subset \realsp^{\fdim \times \nclass}$ is a closed ball of center $0$ and radius $r$.
After this, we set $A_r \defeq Ar^2$ for simplicity.

For $\phi \in L_2^\fdim(\tpr_X)$, we set $w_\phi \defeq \argmin_{w \in \realsp^{\fdim \times \nclass}} \regrisk(\phi,w)$
and $w_{\ndata,\phi} \defeq \argmin_{w \in \realsp^{\fdim\times \nclass} } \regrisk_{\ndata}(\phi,w)$.
Moreover, we define the following notations: 
\begin{align*}
  \nabla_\phi \regrisk(\phi)(x) &\defeq \expec_{\tpr(Y|x)}[ \partial_{z}l(\phi(x),Y,w_\phi) ],  \\
  \nabla_\phi \regrisk_{\ndata}(\phi)(x) & \defeq  
 \begin{cases}
   \partial_{z}l(\phi(x_i),y_i,w_{\ndata,\phi}) & (x=x_i), \\
    0 & (\mathrm{otherwise}). 
  \end{cases} 
\end{align*}
We also similarly define functional gradients $\partial_\phi \regrisk(\phi,w)$ and $\partial_\phi \regrisk_\ndata(\phi,w)$ for fixed $w$ 
by replacing $w_\phi, w_{\ndata,\phi}$ by $w$.
It follows that 
\begin{align*}
  \nabla_\phi \regrisk(\phi) = \partial_\phi \regrisk(\phi,w_\phi),
  \nabla_\phi \regrisk_\ndata(\phi) = \partial_\phi \regrisk_\ndata(\phi,w_{\ndata,\phi}).
\end{align*}

The next proposition means that the above maps are functional gradients in $L_2^\fdim(\tpr_X)$ and $L_2^\fdim(\tpr_{\ndata,X})$.
We set $l_0 = \max_{y\in \labelsp} l(0,y)$.

\begin{proposition} \label{taylor_prop}
  Let Assumption \ref{smoothness_assumption_1} hold.
  Then, for $\forall \phi, \psi \in L_2^\fdim(\tpr_X)$, it follows that
  \begin{equation}
    \regrisk(\psi) = \regrisk(\phi) + \pd< \nabla_\phi \regrisk(\phi),\psi-\phi>_{L_2^\fdim(\tpr_X)} + H_\phi(\psi), \label{differential_eq}
  \end{equation}
  where $H_\phi(\psi) \leq \frac{A_{c_\lambda}}{2}\| \phi-\psi \|_{L_2^\fdim(\tpr_X)}^2$ $(c_\lambda=\sqrt{2l_0/\lambda})$.
  Furthermore, the corresponding statements hold for $\regrisk(\cdot,w)$ $(\forall w \in \realsp^\fdim)$ by replacing $\regrisk(\cdot)$ by $\regrisk(\cdot,w)$
  and for empirical variants by replacing $\tpr_X$ by $\tpr_{\ndata,X}$.
\end{proposition}

We can also show differentiability of $\risk(f)$ and $\risk_\ndata(f)$.
Their functional gradients have the form
$\nabla_f \risk(f)(x)=\expec_{\tpr(Y|x)}[\partial_\zeta l(f(x),Y)]$ and $\nabla_f \risk_\ndata(f)(x_i)= \partial_\zeta l(f(x_i),y_i)$.
In this paper, we derive functional gradient methods using $\nabla_\phi \regrisk_\ndata(\phi)$ rather than $\nabla_f \risk_\ndata(f)$
like usual gradient boosting \cite{mason1999boosting, friedman2001greedy}, and provide convergence analyses
for problems (\ref{original_expected_risk_prob}) and (\ref{original_empirical_risk_prob}).
However, we cannot apply $\nabla_\phi \regrisk_{\ndata}(\phi)$ or $\partial_\phi \regrisk_{\ndata}(\phi,w)$ directly to the expected risk minimization problem
because these functional gradients are zero outside the training data.
Thus, we need a smoothing technique to propagate these to unseen data.
The expected benefit of functional gradient methods using $\nabla_\phi \regrisk_\ndata(\phi)$ over usual gradient boosting is that
the former can learn a deep model that is known to have high representational power.
Before providing a concrete algorithm description, we first explain the basic property of functional gradients and functional gradient methods.

\section{Basic Property of Functional Gradient} \label{functional_gradient_prop_sec}
In this section, we explain the motivation for using functional gradients for solving classification problems.
We first show the consistency of functional gradient norms, namely predicted probabilities by predictors with small norms converge to empirical/expected conditional probabilities.
We next explain the superior performance of functional gradient methods intuitively, which motivate us to use it for finding predictors with small norms.
Moreover, we explain that the optimization procedure of functional gradient methods can be realized by stacking ResNet layers iteratively on the top of feature extractions.

\paragraph{Consistency of functional gradient norm.}
We here provide upper bounds on the gaps between true empirical/expected conditional probabilities and predicted probabilities.

\begin{proposition} \label{consistency_prop}
  Let $l(\zeta,y)$ be the loss function for the multiclass logistic regression.
  Then, 
  \begin{align*}
    &\| \nabla_f \risk(f) \|_{L_1^\nclass(\tpr_X)}
      \geq \frac{1}{\sqrt{\nclass}}\sum_{y\in \labelsp} \| \tpr(y|\cdot) - p_f(y|\cdot) \|_{L_1(\tpr_X)}, \\
    &\| \nabla_f \risk_\ndata(f) \|_{L_1^\nclass(\tpr_{\ndata,X})}
      \geq \frac{1}{\sqrt{\nclass}}\sum_{y\in \labelsp} \| \tpr_{\ndata}(y|\cdot) - p_f(y|\cdot) \|_{L_1(\tpr_{\ndata,X})},
  \end{align*}
  where we denote by $p_f(y|x)$ the softmax function defined by the predictor $f$, i.e., $\exp(f_{y}(\cdot))/\sum_{\overline{y}\in \labelsp}\exp(f_{\overline{y}}(\cdot))$.
\end{proposition}

Many studies \cite{zhang2004statistical, steinwart2005consistency, bartlett2006convexity} have exploited the consistency of convex loss functions for classification problems
in terms of the classification error or conditional probability.
Basically, these studies used the excess empirical/expected risk to estimate the excess classification error or the gap between the true conditional probability and the predicted probability.
On the other hand, Proposition \ref{consistency_prop} argues that functional gradient norms give sufficient bounds on such gaps.
This fact is very helpful from the optimization perspective for non-strongly convex smooth problems since the excess risk always bounds the functional gradient norm by the reasonable order, 
but the inverse relationship does not always hold.
This means that finding a predictor with a small functional gradient is much easier than finding a small excess risk.

Note that the latter inequality in Proposition \ref{consistency_prop} provides the lower bound on empirical classification accuracy, which is confirmed by Markov inequality as follows.
\begin{align*}
  \prob_{\tpr_\ndata}[ 1-p_f(Y|X) \geq 1/2]
  &\leq 2 \expec_{\tpr_\ndata}[1-p_f(Y|X)] \notag\\
  &\leq 2\sqrt{\nclass} \| \nabla_f \risk_\ndata(f) \|_{L_1^\nclass(\tpr_{\ndata,X})}. 
\end{align*}
Generally, we can derive a bound on the empirical margin distribution \cite{koltchinskii2002empirical} by using the functional gradient norm in a similar way,
and can obtain a generalization bound using it, as shown later.

\paragraph{Powerful optimization ability and connection to residual networks.}
In the above discussion, we have seen that the classification problem can be reinterpreted as finding a predictor with small functional gradient norms,
which may lead to reasonable convergence compared to minimizing the excess risk.
However, finding such a good predictor is still difficult because a function space is quite comprehensive, and thus, a superior optimization method is required
to achieve this goal.
We explain that functional gradient methods exhibit an excellent performance by using the simplified problem. 
Namely, we apply the functional gradient method to the following problem:

\begin{equation}
  \min_{\phi \in L_2^\fdim(\tpr_X)} \left\{ \dumrisk(\phi) \defeq
  \expec_{\tpr_X}[r(\phi(X))] \right\}, \label{dummy_expected_risk_prob}  
\end{equation}
where $r$ is a sufficiently smooth function.
Note that the main problem (\ref{expected_risk_prob}) is not interpreted as this simplified problem correctly,
but is useful in explaining a property and an advantage of the method and leads to a deeper understanding.

If $\dumrisk$ is Fr\'{e}chet differentiable, the functional gradient is represented as $\nabla_\phi \dumrisk(\phi)(\cdot)=\nabla_z r(\phi(\cdot))$, where
$z$ indicates the input to $r$.
Therefore, the negative functional gradient indicates the direction of decreasing the objective $r$ at each point $\phi(x)$.
An iteration of the functional gradient method with a learning rate $\eta$ is described as
\[ \phi_{t+1} \leftarrow \phi_t - \eta \nabla_z r\circ \phi_t = (id-\eta \nabla_zr) \circ \phi_t. \]
We can immediately notice that this iterate makes $\phi_t$ one level deeper by stacking a residual network-type layer $id-\eta \nabla_zr$ \cite{he2016deep},
and data representation is refined through this layer.
Starting from a simple feature extraction $\phi_0$ and running the functional gradient method for $T$-iterations, we finally obtain a residual network:
\[ \phi_T = (id-\eta \nabla_zr)\circ \cdots \circ (id-\eta \nabla_zr)\circ \phi_0. \]
Therefore, feature extraction $\phi$ gradually grows by using the functional gradient method to optimize $\dumrisk$.
Indeed, we can show the convergence of $\phi_T$ to a stationary point of $\dumrisk$ in $L_2^\fdim(\tpr_X)$ under smoothness and boundedness assumptions
by analogy with a finite-dimensional optimization method.
This is a huge advantage of the functional gradient method because stationary points in $L_2^\fdim(\tpr_X)$ are potentially significant better than
those of finite-dimensional spaces.
Note that this formulation explains the optimization view \cite{jastrzebski2017residual} of ResNet mathematically.

We now briefly explain how powerful the functional gradient method is compared to the gradient method in a finite-dimensional space, for optimizing $\dumrisk$.
Let us consider any parameterization of $\phi_t \in L_2^\fdim(\tpr_X)$.
That is, we assume that $\phi_t$ is contained in a family of parameterized feature extractions
$\mathcal{M}=\{g_{\theta} : \featuresp \rightarrow \featuresp \mid \theta \in \Theta \subset \realsp^m\} \subset L_2^\fdim(\tpr_X)$,
i.e., $\exists \theta' \in \Theta\ s.t.\ \phi_t=g_{\theta'}$.
Typically, the family $\mathcal{M}$ is given by neural networks.
If $\dumrisk(g_\theta)$ is differentiable at $\theta'$,
we get $\nabla_\theta \dumrisk(g_\theta)|_{\theta=\theta'}=\pd< \nabla_\phi \dumrisk(\phi_t), \nabla_\theta g|_{\theta=\theta'}>_{L_2^\fdim(\tpr_X)}$
according to the chain rule of derivatives.
Note that $\nabla_\phi \dumrisk(\phi_t)$ dominates the norm of gradients.
Namely, if $\phi_t$ is a stationary point in $L_2^\fdim(\tpr_X)$, $\phi_t$ is also a stationary point in $\mathcal{M}$ and there is no room for improvement using gradient-based methods.
This result holds for any family $\mathcal{M}$, but the inverse relation does not always hold.
This means that gradient-based methods may fail to optimize $\dumrisk$ in the function space,
while the functional gradient method exceeds such a limit by making a feature extraction $\phi_t$ deeper.
For detailed descriptions and related work in this line, we refer to \citet{ags2008,daneri2010lecture,sonoda2017double,nitanda2017stochastic,nitanda2018gradient}.

\section{Algorithm Description} \label{algorithm_description_sec}
In this section, we provide concrete description of the proposed method.
Let $\phi_t \in L_2^\fdim(\tpr_X)$ and $w_t$ denote $t$-th iterates of $\phi$ and $w$. 
As mentioned above, since functional gradients $\partial_\phi \regrisk_{\ndata}(\phi_t,w_{t+1})$ for the empirical risk vanish outside the training data,
we need a smoothing technique to propagate these to unseen data.
Hence, we use the convolution $T_{k_t,\ndata}\partial_\phi \regrisk_{\ndata}(\phi_t,w_{t+1})$ of the functional gradient
by using an adaptively chosen kernel function $k_t$ on $\featuresp$.
The convolution is applied element-wise as follows.
\begin{align*}
  T_{k_t,\ndata}\partial_\phi \regrisk_{\ndata}(\phi_t,w_{t+1})
  &\defeq \expec_{\tpr_{n,X}}[ \partial_\phi \regrisk_{\ndata}(\phi_t,w_{t+1})(X) k_t(X,\cdot)] \notag \\
  &= \frac{1}{\ndata}\sum_{i=1}^{\ndata}\partial_z l(\phi_t(x_i),y_i,w_{t+1}) k_t(x_i,\cdot). 
\end{align*}

Namely, this quantity is a weighted sum of $\partial_\phi \regrisk_{\ndata}(\phi_t,w_{t+1})(x_i)$ by $k_t(x_i,\cdot)$, which we also call a functional gradient.
In particular, we restrict the form of a kernel $k_t$ to the inner-product of a non-linear feature embedding to a finite-dimensional space by $\iota_t: \realsp^{\fdim}\rightarrow \realsp^{\edim}$,
that is, $k_t(x,x')=\dotp{\iota_t(\phi_t(x))}{\iota_t(\phi_t(x))}$.
The requirements on the choice of $\iota_t$ to guarantee the convergence are the uniform boundedness and sufficiently preserving
the magnitude of the functional gradient $\partial_\phi \regrisk_{\ndata}(\phi_t,w_{t+1})$.
Let $\mathcal{F}$ be a given restricted class of bounded embeddings.
We pick up $\iota_t$ from this class $\mathcal{F}$ by approximately solving the following problem to acquire magnitude preservation:
\begin{equation}
   \max_{\iota_t \in \mathcal{F}} \|T_{k_t,\ndata}\partial_\phi \regrisk_\ndata(\phi_t,w_{t+1})\|_{k_t}^2. \label{resblock_prob}
\end{equation}

where we define
$\|T_{k_t,\ndata}\xi\|_{k_t}^2 = \pd<\xi,T_{k_t,\ndata}\xi>_{L_2^\fdim(\tpr_{\ndata,X})}$ for a vector function $\xi$.
Detailed conditions on $\iota_t$ and an alternative problem to guarantee the convergence will be discussed later. 
Note that due to the restriction on the form of $k_t$, the computation of the functional gradient
is compressed to the matrix-vector product.
Namely, 
\begin{align*}
  &A_t\defeq \frac{1}{\ndata}\sum_{i=1}^{\ndata}\partial_\phi\regrisk_{\ndata}(\phi_t,w_{t+1})(x_i) \iota_t(\phi_t(x_i))^{\top}, \\
  &T_{k_t,\ndata}\partial_\phi \regrisk_{\ndata}(\phi_t,w_{t+1}) = A_t \iota_t( \phi_t(\cdot)).
\end{align*}
Therefore, the functional gradient method $\phi_{t+1} \leftarrow \phi_t - \eta_t T_{k_t,\ndata}\partial_\phi \regrisk_{\ndata}(\phi_t,w_{t+1})$
can be recognized as the procedure of successively stacking layers
$id-\eta_t A_t \iota_t( \phi_t(\cdot))$ ($t \in \{0,\ldots,T-1\}$) and obtaining a residual network.
The entire algorithm is described in Algorithm \ref{PFGD}.
Note that because a loss function $l$ is chosen typically to be convex with respect to $w$, a procedure in Algorithm \ref{PFGD}
to obtain $w_{\ndata,\phi_t}$ is easily achieved by running an efficient method for convex minimization problems.
The notation $T_0$ is the stopping time of iterates with respect to $w$.
That is, functional gradients $\partial_\phi \regrisk_\ndata(\phi_t,w_{t+1})$ are computed at $w_{t+1}=w_{\ndata,\phi_t}$ and correspond to $\nabla_\phi \regrisk_\ndata(\phi_t)$ when $t < T_0$
and computed at an older point of $w$ when $t \geq T_0$, rather than $\nabla_\phi \regrisk_\ndata(\phi_t)$.

\begin{algorithm}[h]
  \caption{ResFGB}
  \label{PFGD}
\begin{algorithmic}
  \STATE {\bfseries Input:} $S=(x_i,y_i)_{i=1}^{\ndata}$, initial points $\phi_0,\ w_0$,
  the number of iterations $T$ of $\phi$, the number of iterations $T_0$ of $w$, embedding class $\mathcal{F}$, and learning rates $\eta_t$\\
   \vspace{1mm}
   
   \FOR{$t=0$ {\bfseries to} $T-1$}
   \IF{$t < T_0$}
   \STATE $w_{t+1} \leftarrow w_{\ndata,\phi_t}=\argmin_{w \in \realsp^{\fdim \times \nclass}} \regrisk_{\ndata}(\phi_t,w)$
   \ELSE
   \STATE $w_{t+1} \leftarrow w_{t}$   
   \ENDIF
   \STATE Get $\iota_t$ by approximately solving (\ref{resblock_prob}) on $S$\\   
   \STATE $A_t \leftarrow \frac{1}{\ndata}\sum_{i=1}^{\ndata} \partial_z l(\phi_t(x_i),y_i,w_{t+1}) \iota_t(\phi_t(x_i))^{\top}$ \\
   \STATE $\phi_{t+1} \leftarrow \phi_t - \eta_t A_t \iota_t( \phi_t(\cdot) )$
   \ENDFOR
   \STATE Return $\phi_{T-1}$ and $w_T$
\end{algorithmic}
\end{algorithm}

\paragraph{Choice of embedding.} 
We here provide policies for the choice of $\iota_t$.
A sufficient condition for $\iota_t$ to achieve good convergence is to maintain the functional gradient norm, 
which is summarized below.

\begin{assumption} \label{kernel_choice_assumption}
  For positive values $\gamma$, $\epsilon$, $p \leq 2$, $q$, and $K$, a function $k_t(x,x')=\dotp{\iota_t(\phi_t(x))}{\iota_t(\phi_t(x))}$ satisfies
  $\|\iota_t(x)\|_2 \leq \sqrt{K}$ on $\featuresp$, and 
  $\gamma\|\partial_\phi \regrisk_\ndata(\phi_t,w_{t+1}) \|_{L_p^{\fdim}(\tpr_{\ndata,X})}^{q} - \gamma \epsilon
  \leq \|T_{k_t,\ndata}\partial_{\phi} \regrisk_\ndata(\phi_t,w_{t+1})\|_{k_t}^2$.
\end{assumption}

This assumption is a counterpart of that imposed in \citet{mason1999boosting}.
The existence of $\iota_t$, not necessarily included in $\mathcal{F}$, satisfying this assumption is confirmed as follows.
We here assume that $\phi_t$ is a bijection that is a realistic assumption when learning rates are sufficiently small because of the inverse mapping theorem.
Then, since $\tpr_{\ndata}(\cdot | X)=\tpr_{\ndata}(\cdot | \phi_t(X))$, functional gradients $\partial_\phi\regrisk_\ndata(\phi_t,w_{t+1})(x)$
become the map of $\phi_t(x)$, so we can choose $\iota_t$ such that
\[ \iota_t(\phi_t(\cdot)) = \partial_\phi\regrisk_\ndata(\phi_t,w_{t+1})(\cdot)/\|\partial_\phi\regrisk_\ndata(\phi_t,w_{t+1})(\cdot)\|_2. \]
By simple computation, we find that $k_t(x,x')\leq 1$ and $\|T_{k_t,\ndata}\partial_{\phi} \regrisk_\ndata(\phi_t,w_{t+1})\|_{k_t}^2$
are lower-bounded by $\frac{1}{\fdim}\|\partial_\phi \regrisk_\ndata(\phi_t,w_{t+1})\|_{L_1^\fdim(\tpr_{\ndata,X})}^2$.
A detailed derivation is provided in Appendix.
Thus, Assumption \ref{kernel_choice_assumption} may be satisfied if an embedding class $\mathcal{F}$ is sufficiently large,
but we note that too large $\mathcal{F}$ leads to overfitting.
Therefore, one way of choosing $\iota_t$ is to approximate $\partial_\phi\regrisk_\ndata(\phi_t,w_{t+1})(\cdot)/\|\partial_\phi\regrisk_\ndata(\phi_t,w_{t+1})(\cdot)\|_2$ rather than
maximizing (\ref{resblock_prob}) directly, and indeed, this procedure has been adopted in experiments.

\section{Convergence Analysis} \label{conv_analysis_sec}
In this section, we provide a convergence analysis for the proposed method.
All proofs are included in Appendix.
For the empirical risk minimization problem, we first show the global convergence rate, which also provides the generalization bound by combining the standard complexity analyses.
Next, for the expected risk minimization problem, we describe how the size of $\mathcal{F}$ and the learning rate control the tradeoff between optimization speed and generalization
by using the sample-splitting variant of Algorithm \ref{PFGD}, whose detailed description will be provided later.

\paragraph{Empirical risk minimization.}
Using Proposition \ref{taylor_prop}, Assumption \ref{kernel_choice_assumption}, and an additional assumption on $w_t$,
we can show the global convergence of Algorithm \ref{PFGD}.
The following inequality shows how functional gradients decrease the objective function, which is a direct consequence of Proposition \ref{taylor_prop}.
When $\eta \leq \frac{1}{A_{c_\lambda}K}$, we have
\begin{align*}
  \regrisk_\ndata(\phi_{t+1},w_{t+2}) &\leq \regrisk_\ndata(\phi_t,w_{t+1}) \\
  & - \frac{\eta}{2}\|T_{k_t,\ndata}\partial_{\phi} \regrisk_\ndata(\phi_t,w_{t+1})\|_{k_t}^2.  
\end{align*}
Therefore, Algorithm \ref{PFGD} provides a certain decrease in the objective function; moreover, we can conclude a stronger result.

\begin{theorem} \label{convergence_thm}
  Let Assumptions \ref{smoothness_assumption_1} and \ref{kernel_choice_assumption} hold.
  Consider running Algorithm \ref{PFGD} with a constant learning rate $\eta_t = \eta \leq \frac{1}{A_{c_\lambda}K}$.
  If $p\geq 1$ and the minimum eigenvalues of $(w_t{^\top}w_t)_{t=0}^{T_0}$ have a uniform lower bound $\sigma^2>0$, then 
  \begin{equation}
    \frac{1}{T}\sum_{t=0}^{T-1} \| \nabla_f \risk_\ndata(f_t)\|_{L_1^\nclass(\tpr_{\ndata,X})}^{q}
    \leq \frac{2\regrisk_{\ndata}(\phi_0,w_1)}{\eta\gamma\sigma^qT} + \frac{\epsilon}{\sigma^q} \label{conv_rate}
  \end{equation}
  where we denote $f_t = w_{t+1}^{\top}\phi_t$.
\end{theorem}

\paragraph{Remark.}
(i) This theorem states the convergence of the average of functional gradient norms obtained by running Algorithm \ref{PFGD},
  but we note that it also leads to the convergence of the minimum functional gradient norms.
(ii) Although a larger value of $T_0$ may affect the bound in Theorem \ref{convergence_thm} because of dependency on the minimum eigenvalue of $(w_t^{\top}w_t)_{t=0}^{T_0}$,
optimizing $w$ at each iteration facilitates the convergence speed empirically.

Theorem \ref{convergence_thm} means that the convergence becomes faster when an input distribution has the high degree of linear separability.
However, even when it is somewhat large, a much faster convergence rate in the second half of the algorithm is achieved
by making an additional assumption where loss function values attained by the algorithm are uniformly bounded.

\begin{theorem} \label{fast_convergence_thm}
  Let Assumptions \ref{smoothness_assumption_1} and \ref{kernel_choice_assumption} with $(\epsilon,p,q)=(0,1,2)$ hold.
  We assume $T/2 \in \mathbb{N}$ for simplicity.
  Consider running Algorithm \ref{PFGD} with learning rates $\eta_0$ and $\eta_1$ in the first half and the second half of Algorithm, respectively.
  We assume $\eta_0, \eta_1 \leq \frac{\gamma}{Ac_\lambda^2K^2}$.
  We set $f_t = w_{t+1}^{\top}\phi_t$.
  Moreover, assume that there exists $\exists M>0$ such that $l(f_t(X),Y) \leq M$ for $(X,Y)\sim \tpr_{\ndata,\featuresp}$ 
  and the minimum eigenvalues of $(w_t{^\top}w_t)_{t=0}^{T_0}$ have a uniform lower bound $\sigma^2>0$.
  Then we get
  \begin{align*}
  \frac{1}{T}\sum_{t=0}^{\frac{T}{2}-1} \|\nabla_f\risk_{\ndata}(f_t) \|_{L_1^\nclass(\tpr_{\ndata,X})}^2
  &\leq \frac{2\risk_n(f_0)}{\eta_0\gamma \sigma^2T}, \\
  \frac{1}{T}\sum_{t=\frac{T}{2}}^{T-1} \|\nabla_f\risk_{\ndata}(f_t) \|_{L_1^\nclass(\tpr_{\ndata,X})}^2
  &\leq \frac{4\risk_\ndata(f_0)}{\eta_1\gamma \sigma^2(2+\eta_0\alpha \risk_\ndata(f_0)T)T}.
  \end{align*}
\end{theorem}

\vspace{-3mm}
\paragraph{Generalization bound.}
Here, we derive a generalization bound using the margin bound developed by \citet{koltchinskii2002empirical},
which is composed of the sum of the empirical margin distribution and Rademacher complexity of predictors. 
The margin and the empirical margin distribution for multiclass classification are defined as $m_f(x,y)\defeq f_y(x)-\max_{y'\neq y}f_{y'}(x)$
and $\prob_{\tpr_{\ndata}}[m_f(x,y)\leq \delta]$ ($\delta>0$), respectively.
When $l$ is the multiclass logistic loss, using Markov inequality and Proposition \ref{consistency_prop}, we can obtain an upper bound on the margin distribution:
\[ \prob_{\tpr_{\ndata}}[m_f(x,y)\leq \delta] \leq \left(1+\frac{1}{\exp(-\delta)}\right) \sqrt{\nclass} \| \nabla_f \risk_\ndata(f) \|_{L_1^\nclass(\tpr_{\ndata,X})}. \]

Since the convergence of functional gradient norms has been shown in Theorem \ref{convergence_thm} and \ref{fast_convergence_thm},
the resulting problem to derive a generalization bound is to estimate Rademacher complexity, which can be achieved using standard techniques
developed by \citet{bartlett2002rademacher, koltchinskii2002empirical}.
Thus, we specify here the architecture of predictors.
In the theoretical analysis, we suppose $\mathcal{F}$ is the set of shallow neural networks $B\sigma(Cx)$ for simplicity, where $B, C$ are weight matrices and $\sigma$ is
an element-wise activation function.
Then, the $t$-th layer is represented as
\[ \phi_{t+1}(x) = \phi_{t}(x) - D_t \sigma(C_t \phi_t(x)), \]
where $D_t = \eta_tA_tB_t$, and a predictor is $f_{T-1}(x)=w_T^{\top}\phi_{T-1}(x)$.
Bounding norms of these weights by controlling the size of $\mathcal{F}$ and $\lambda$, we can restrict the Rademacher complexity of a set of predictors and obtain a generalization bound.
We denote by $\mathcal{G}_{T-1}$ the set of predictors under constraints on weight matrices where $L_1$-norms of each row of $w_{T}^{\top}, C_t$, and $D_t$ are bounded by
$\Lambda_{w}, \Lambda$, and $\Lambda'_t$.
\begin{align*} \mathcal{G}_{T-1} \defeq \{ &\|(w_{T})_{*,y}\|_{1} \leq \Lambda_w,\ \|(C_{t})_{i,*}\|_1 \leq \Lambda, \\
  &\|(D_{t})_{j,*}\|_1 \leq \Lambda'_t,\ t\in \{0,\ldots,T-1\},\ \forall y,\forall i,\forall j \}.
\end{align*}
\begin{theorem} \label{margin_bound_thm}
  Let $l$ be the multiclass logistic regression loss.
  Fix $\delta > 0$.
  Suppose $\sigma$ is $L_\sigma$-Lipschitz continuous and $\|x\|_{2} \leq \Lambda_{\infty}$ on $\featuresp$.
  Then, for $\forall \rho > 0$, with probability at least $1-\rho$ over the random choice of $S$ from $\tpr^\ndata$, we have $\forall f \in \mathcal{G}_{T-1}$,
  \begin{align*}
    \prob_{\tpr}[m_f(X,&Y) \leq 0]
    \leq \frac{2\nclass^3\Lambda_{\infty}\Lambda_w}{\delta\sqrt{\ndata}}\prod_{t=0}^{T-2}(1+\Lambda\Lambda'_tL_\sigma)\\
    &\hspace{-16mm}+ \sqrt{\frac{\log(1/\rho)}{2\ndata}}
     +\left(1+\frac{1}{\exp(-\delta)}\right) \sqrt{\nclass} \| \nabla_f \risk_\ndata(f) \|_{L_1^\nclass(\tpr_{\ndata,X})}. 
  \end{align*}  
\end{theorem}

Combining Theorems \ref{convergence_thm}, \ref{fast_convergence_thm} and \ref{margin_bound_thm}, we observe that the learning rates $\eta_t$, the number of iterations $T$, and the size of $\mathcal{F}$
have an impact on the optimization-generalization tradeoff, that is, larger values of these quantities facilitate the convergence on training data
while the generalization bound becomes gradually loose.
Especially, this bound has an exponential dependence on depth $T$, which is known to be unavoidable \cite{neyshabur2015norm}
in the worst case for some networks with $L_1$ or the group norm constraints, but this bound is useful when an initial objective is small and required $T$ is also small sufficiently.

We next derive an interesting bound for explaining the effectiveness of the proposed method.
This bound can be obtained by instantiating bounds in Theorem \ref{margin_bound_thm} for various $T,\ \Lambda_t'$ and making an union bound.
Since norms of rows of $A_t$ are uniformly bounded by their construction, norm constraints on $D_t=\eta_tA_tB_t$ is reduced to bounding a norm of $B_t$.
Thus, we further assume $\sum_l \|(B_t)_{*,l}\|_2 \leq \Lambda''$. 
\begin{corollary} \label{union_bound}
  Let $l$ be the multiclass logistic regression loss.
  Fix $\delta > 0$.
  Suppose $\sigma$ is $L_\sigma$-Lipschitz continuous and $\|x\|_{2} \leq \Lambda_{\infty}$ on $\featuresp$.
  Then, for $\forall \rho > 0$, with probability at least $1-\rho$ over the random choice of $S$ from $\tpr^\ndata$, the following bound is valid 
  for any function $f_{T-1}$ obtained by Algorithm \ref{PFGD} under constraints $\|(w_{T})_{*,y}\|_{1} \leq \Lambda_w$, $\sum_l \|(B_t)_{*,l}\|_2 \leq \Lambda''$, and $\|\iota_t(x)\|_2 \leq \sqrt{K}$.
  \begin{align*}
      &\prob_{\tpr}[m_{f_{T-1}}(X,Y) \leq 0]
    \leq  \frac{2\nclass^3\Lambda L_\sigma\Lambda_{\infty}\Lambda_w}{\delta\sqrt{\ndata}}\\
    &+ \frac{\nclass^3\Lambda_{\infty}\Lambda_w}{\delta\sqrt{\ndata}}\left(1+\frac{C}{T-1}\sum_{t=0}^{T-2}  \eta_t \| \nabla_f \risk_\ndata(f_t) \|_{L_1^\nclass(\tpr_{\ndata,X})}  \right)^{T-1} \\    
    &+ \sqrt{\frac{1}{2\ndata}\left(\log\left(\frac{1}{\rho}\right) + O(T\log T)\right) } \\
    &+ \left(1+\frac{1}{\exp(-\delta)}\right) \sqrt{\nclass} \| \nabla_f \risk_\ndata(f_{T-1}) \|_{L_1^\nclass(\tpr_{\ndata,X})},
  \end{align*}
  where $f_t=w_{t+1}^\top\phi_t$ and $C=2\Lambda L_\sigma \sqrt{K\fdim} \Lambda'' c_\lambda$.
\end{corollary}
This corollary shows an interesting and useful property of our method in terms of generalization, that is,
fast convergence of functional gradient norms leads to small complexity of an obtained network, surprisingly.
As a result, our method is expected to get a network with good generalization because it directly minimizes functional gradient norms.

By plugging in convergence rates of functional gradient norms in Theorem \ref{convergence_thm} and \ref{fast_convergence_thm} for the generalization bound in Corollary \ref{union_bound},
we can obtain explicit convergence rates of classification errors.
For instance, under the assumption in Theorem \ref{convergence_thm} with $q=2,\ \epsilon=0$, and a learning rate $\eta = O(1/T^\alpha)$ $0 \leq \alpha < 1$, then the generalization bound
becomes
\[ O\left( \frac{1}{\sqrt{n}}\left( \exp(T^{\frac{1-\alpha}{2}}) + \sqrt{\log\frac{1}{\rho}} \right) + \frac{\regrisk_\ndata(\phi_0)}{T^{\frac{1-\alpha}{2}}} \right). \]
Moreover, under the assumption in Theorem \ref{fast_convergence_thm} with learning rates $\eta_0 = O(1/T^\alpha)$ and $\eta_1 = O(1/T^{2\alpha - 1})$ $\frac{1}{2} \leq \alpha < 1$,
a faster convergence rate is achieved.
\[ O\left( \frac{1}{\sqrt{n}}\left( \exp(T^{\frac{1-\alpha}{2}}) + \sqrt{\log\frac{1}{\rho}} \right) + \frac{1}{T^{\frac{3(1-\alpha)}{2}}} \right). \]
Note that by utilizing the corollary, the optimization and generalization tradeoff depending on the number of iterations and learning rates is confirmed more clearly.

We note another type of bound can be derived by utilizing VC-dimension or pseudo-dimension \cite{vapnik1971uniform}.
When the activation function is piece-wise linear, such as Relu function $\sigma(x)=\max\{0,x\}$, reasonable bounds on these quantities are given by \citet{bartlett1998almost,1703.02930}.
Thus, for that case, we can obtain better bounds with respect to $T$ by combining our analysis and the VC bound,
but we omit the precise description for simplicity.
We next show the other generalization guarantee from the optimization perspective by using the modified algorithm,
which may slow down the optimization speed but alleviates the exponential dependence on $T$ in the generalization bound.

\paragraph{Sample-splitting technique.}
To remedy the exponential dependence on $T$ of the generalization bound,
we introduce the sample-splitting technique which has been used recently to provide statistical guarantee of expectation-maximization algorithms \cite{balakrishnan2017statistical,wang2015high}.
That is, instead of Algorithm \ref{PFGD}, we analyze its sample-splitting variant.
Although Algorithm \ref{PFGD} exhibits good empirical performance, the sample-splitting variant is useful for analyzing the behavior of the expected risk.
In this variant, the entire dataset is split into $T$ pieces, where $T$ is the number of iterations, 
and each iteration uses a fresh batch of samples.
The key benefit of the sample-splitting method is that it allows us to use concentration inequalities independently at each iterate $\phi_t$
rather than using the complexity measure of the entire model.
As a result, sample-splitting alleviates the exponential dependence on $T$ presented in Theorem \ref{margin_bound_thm}.
We now present the details in Algorithm \ref{SS_PFGD}.
For simplicity, we assume $T_0=0$, namely the weight vector $w_t$ is fixed to the initial weight $w_0$.

\begin{algorithm}[h]
  \caption{Sample-splitting ResFGB}
  \label{SS_PFGD}
\begin{algorithmic}
  \STATE {\bfseries Input:} $S=(x_i,y_i)_{i=1}^{\ndata}$, initial points $\phi_0,\ w_0$,
  the number of iterations $T$, embedding class $\mathcal{F}$, and learning rates $\eta$\\
  \vspace{1mm}
  
   \STATE Split $S$ into $T$ disjoint subsets $S_1,\ldots,S_T$ of size $\floor[\ndata/T]$
   \FOR{$t=0$ {\bfseries to} $T-1$}
   \STATE Define $\regrisk_{\floor[\ndata/T]}(\phi_t,w)$ using $S_t$
   \STATE Get $\iota_t$ by approximately solving (\ref{resblock_prob}) on $S_t$\\      
   \STATE $A_t \leftarrow \floor[\frac{T}{n}]\sum_{i=1}^{\floor[\ndata/T]} \partial_z l(\phi_t(x_i),y_i,w_{0}) \iota_t(\phi_t(x_i))^{\top}$ \\
   \STATE $\phi_{t+1} \leftarrow \phi_t - \eta A_t \iota_t( \phi_t(\cdot) )$
   \ENDFOR
   \STATE Return $\phi_{T-1}$ and $w_{0}$
\end{algorithmic}
\end{algorithm}

Our proof mainly relies on bounding a statistical error of the functional gradient at each iteration in Algorithm \ref{SS_PFGD}.
Because the population version of Algorithm \ref{PFGD} strictly decreases the value of $\regrisk$ due to its smoothness,
we can show that Algorithm \ref{SS_PFGD} also decreases it with high probability
when the norm of a functional gradient is larger than a statistical error bound.
Thus, we make here an additional assumption on the loss function to bound the statistical error, which is satisfied for a multiclass logistic loss function.

\begin{assumption} \label{smoothness_assumption_2}
  For the differentiable loss function $l(z,y,w)$ with respect to $z, w$, there exists a positive real number $\beta_r$ depending on $r>0$ such that
  $\| \partial_z l(z,y,w)\|_2 \leq \beta_r$
    for $z \in \featuresp, y \in \labelsp, w \in B_r(0)$.
\end{assumption}

We here introduce the notation required to describe the statement.
We let $\mathcal{F}^j$ be a collection of $j$-th elements of functions in $\mathcal{F}$. 
For a positive value $M$, we set 
\begin{equation*}
  \epsilon(m,\rho) \defeq \beta_{\|w_0\|_2} \sqrt{ \frac{K\fdim\edim}{m} } \left( 2 M + \sqrt{ 2K \log\frac{2\fdim\edim}{\rho} } \right).
\end{equation*}


The following proposition is a key result to bound a statistical error as mentioned above.

\begin{proposition} \label{stat_error_prop}
  Let Assumption \ref{smoothness_assumption_2} hold and each $\mathcal{F}^j$ be the VC-class (for the definition see \citet{vw1996}).
  For $\iota \in \mathcal{F}$, we assume $\|\iota(x)\|_2 \leq \sqrt{K}$ on $\featuresp$.
  We set $\mu$ to be $\tpr_X$ or $\tpr_{m,X}$ and $k(x,x')$ to be $\iota(\phi(x))^\top \iota(\phi(x'))$.
  Then, there exists a positive value $M$ depending on $\mathcal{F}$ and it follows that with probability at least $1-\rho$ over the choice of the sample of size $m$, 
  $\epsilon(m,\rho)$ upper-bounds the following.
  \begin{equation*}
    \sup_{\iota \in \mathcal{F}}\left\| T_{k}\partial_\phi \regrisk(\phi,w_{0}) - T_{k,m}\partial_\phi \regrisk_{m}(\phi,w_{0}) \right\|_{L_2^\fdim(\mu)}.
  \end{equation*}

\end{proposition}

Since each iterate in Algorithm \ref{SS_PFGD} is computed on a fresh batch not depending on previous batches, Proposition \ref{stat_error_prop} can be applied to
all iterates with $m\leftarrow \floor[\ndata/T]$ and $\rho \leftarrow \delta/T$ for $\delta \in (0,1)$.
Thus, when $\floor[\ndata/T]$ is large and $\eta$ is small sufficiently, functional gradients used in Algorithm \ref{SS_PFGD} become good approximation to the population variant,
and we find that the expected risk function is likely to decrease from Proposition \ref{taylor_prop}.
Moreover, we note that statistical errors are accumulated additively rather than the exponential growth.
Concretely, we obtain the following generalization guarantee.

\begin{theorem} \label{expected_convergence_thm}
  Let Assumptions \ref{smoothness_assumption_1}, \ref{kernel_choice_assumption}, and \ref{smoothness_assumption_2}
  and the same assumption in Proposition \ref{stat_error_prop} hold.
  Consider running Algorithm \ref{SS_PFGD}.
  If $p\geq 1$, $\|\partial_\zeta l(\zeta,y)\|_2\leq B$, and the minimum eigenvalue of $w_0{^\top}w_0$ is lower-bounded by $\sigma^2>0$, then we get
  with probability at least $1-\rho$,
  \begin{align*}
    &\| \nabla_f\risk(w_{0}^\top \phi_{t_*})\|_{L_1^\nclass(\tpr_X)}
    \leq B\left( \frac{2T}{\ndata} \log \frac{T}{\rho} \right)^{\frac{1}{4}} 
    + \sqrt{ \frac{B}{\gamma^{\frac{1}{q}}\sigma} } \\
    &\cdot\left\{ \frac{ \regrisk_0 }{ \eta T }  
    + \beta_{\|w_0\|_2}\epsilon\left( \frac{\ndata}{T},\frac{\rho}{T} \right)
    + \frac{\eta}{2} A_{\|w_0\|_2}K^2 \beta_{\|w_0\|_2}^2
    + \gamma \epsilon
    \right\}^{\frac{1}{2q}}
  \end{align*}
  where $\regrisk_0 = \regrisk(w_0,\phi_0)$ and $t_*$ is the index giving the minimum value of
  $\| \nabla_f \risk_{\floor[\ndata/T]}(w_0^\top \phi_t) \|_{L_p^\nclass(\tpr_{\floor[\ndata_T],X}) }$. 
\end{theorem}


\begin{table*}[h]
\caption{Test classification accuracy on binary and multiclass classification.}
\label{comparison_table}
\vskip 0.15in
\begin{center}
  \begin{small}
\begin{sc}
  \begin{tabular}{ccccccc}
    \toprule
Method & 
letter &
usps &
ijcnn1 &
mnist &
covtype &
susy 
\\
\midrule
\multirow{2}{*}{ResFGB (logistic)} & 
{\bf 0.976} &
{\bf 0.953} &
{\bf 0.989} &
0.986 &
0.966 &
{\bf 0.804} \\
 &
{\scriptsize {\bf (0.0019)} } &
{\scriptsize {\bf (0.0007)} }&
{\scriptsize {\bf (0.0004)} }&
{\scriptsize (0.0007) }&
{\scriptsize (0.0004) }&
{\scriptsize {\bf (0.0000)} }\\
\multirow{2}{*}{ResFGB (smooth hinge)} & 
0.975 &
0.952 &
{\bf 0.989} &
{\bf 0.987} &
0.965 &
{\bf 0.804} \\
&
{\scriptsize (0.0014) }&
{\scriptsize (0.0023) }&
{\scriptsize {\bf (0.0005)} }&
{\scriptsize {\bf (0.0010)} }&
{\scriptsize (0.0005) }&
{\scriptsize {\bf (0.0004)} }\\    
\multirow{2}{*}{Multilayer Perceptron} &
0.971 &
0.948 &
0.988 &
0.986 &
0.965 &
{\bf 0.804} \\
&
{\scriptsize (0.0059) }&
{\scriptsize (0.0045) }&
{\scriptsize (0.0010) }&
{\scriptsize (0.0010) }&
{\scriptsize (0.0015) }&
{\scriptsize {\bf (0.0004)} }\\       
\multirow{2}{*}{Support Vector Machine} \rule[0mm]{0mm}{3mm}& 
0.959 &
0.948 &
0.977 &
0.969 &
0.824 &
0.754 \\
&
{\scriptsize (0.0062) }&
{\scriptsize (0.0023) }&
{\scriptsize (0.0015) }&
{\scriptsize (0.0041) }&
{\scriptsize (0.0059) }&
{\scriptsize (0.0534) }\\        
\multirow{2}{*}{Random Forest} \rule[0mm]{0mm}{3mm}& 
0.964 &
0.939 &
0.980 &
0.972 &
0.948 &
0.802 \\
&
{\scriptsize (0.0012) }&
{\scriptsize (0.0018) }&
{\scriptsize (0.0005) }&
{\scriptsize (0.0005) }&
{\scriptsize (0.0005) }&
{\scriptsize (0.0004) }\\        
\multirow{2}{*}{Gradient Boosting} \rule[0mm]{0mm}{3mm}& 
0.964 &
0.938 &
0.982 &
0.981 &
{\bf 0.972} &
{\bf 0.804} \\
&
{\scriptsize (0.0011) }&
{\scriptsize (0.0039) }&
{\scriptsize (0.0010) }&
{\scriptsize (0.0004) }&
{\scriptsize {\bf (0.0005)} }&
{\scriptsize {\bf (0.0005)} }\\ 
\bottomrule
\end{tabular}
\end{sc}
\end{small}
\end{center}
\vskip -0.1in
\end{table*}
\vspace{-3mm}
\section{Experiments}
In this section, we present experimental results of the binary and multiclass classification tasks.
We run Algorithm \ref{PFGD} and compare it with support vector machine, random forest, multilayer perceptron, and gradient boosting methods.
We here introduce settings used for Algorithm \ref{PFGD}.
As for the loss function, we test both multiclass logistic loss and smooth hinge loss, and as for the embedding class $\mathcal{F}$, we use two or three hidden-layer neural networks.
The number of hidden units in each layer is set to $100$ or $1000$.
Linear classifiers and embeddings are trained by Nesterov's momentum method. 
The learning rate is chosen from $\{10^{-3},10^{-2},10^{-1},1\}$.
These parameters and the number of iterations $T$ are tuned based on the performance on the validation set.

\begin{figure}[h]
  \begin{center}
    \includegraphics[width=80mm,angle=0]{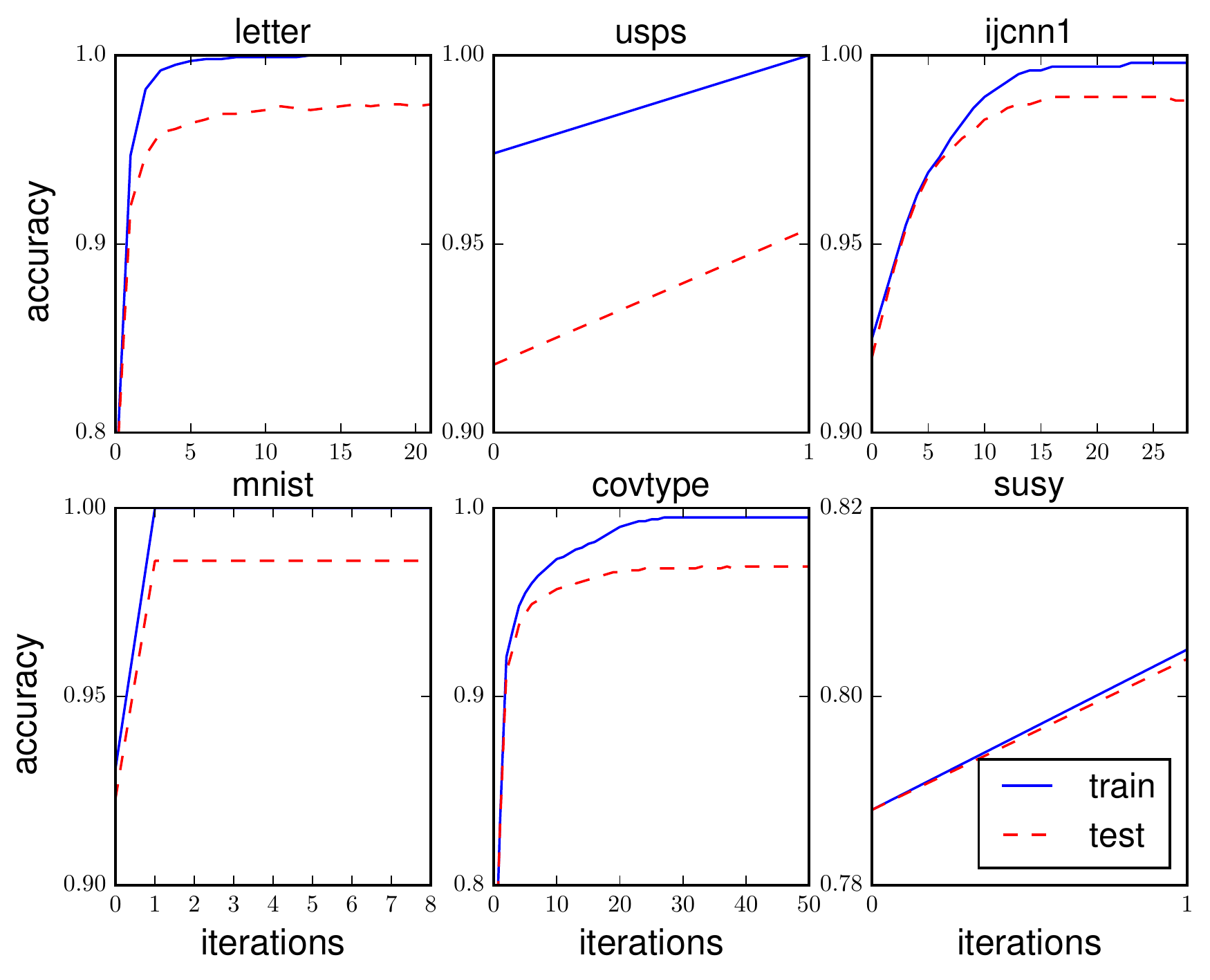}
    \end{center}
    \vspace{-3mm}
    \caption{Learning curves for Algorithm \ref{PFGD} with multiclass logistic loss on libsvm datasets showing classification accuracy on training and test sets 
    versus the number of iterations.} \label{learning_curve}
\end{figure}

We use the following benchmark datasets: letter, usps, ijcnn1, mnist, covtype, and susy.
We now explain the experimental procedure.
For datasets not providing a fixed test set, we first divide each dataset randomly into two parts: $80\%$ for training and the rest for test.
We next divide each training set randomly and use $80\%$ for training and the rest for validation.
We perform each method on the training dataset with several hyperparameter settings and choose the best setting on the validation dataset.
Finally, we train each model on the entire training dataset using this setting and evaluate it on the test dataset.
This procedure is run $5$ times.

The mean classification accuracy and the standard deviation are listed in Table \ref{comparison_table}.
The support vector machine is performed using a random Fourier feature \cite{rahimi2007random} with an embedding dimension of $10^3$ or $10^4$.
For multilayer perceptron, we use three, four, or five hidden layers and rectified linear unit as the activation function.
The number of hidden units in each layer is set to $100$ or $1000$.
As for random forest, the number of trees is set to $100$, $500$, or $1000$ and the maximum depth is set to $10$, $20$, or $30$.
Gradient boosting in Table \ref{comparison_table} indicates LightGBM \cite{ke2017lightgbm} 
with the hyperparameter settings: the maximum number of estimators is $1000$, the learning rate is chosen from $\{10^{-3},10^{-2},10^{-1},1\}$,
and number of leaves in one tree is chosen from  $\{16,32,\ldots,1024\}$.

As seen in Table \ref{comparison_table}, our method shows superior performance over the competitors except for covtype.
However, the method that achieves higher accuracy than our method is only LightGBM on covtype.
We plot learning curves for one run of Algorithm \ref{PFGD} with logistic loss, which depicts classification accuracies on training and test sets.
Note that the number of iterations are determined by classification results on validation sets.
This figure shows the efficiency of the proposed method.
\vspace{-3mm}
\section{Conclusion}
We have formalized the gradient boosting perspective of ResNet 
and have proposed new gradient boosting method by leveraging this viewpoint.
We have shown two types of generalization bounds: one is by the margin bound and the other is by the sample-splitting technique.
These bounds clarify the optimization-generalization tradeoff of the proposed method.
Impressive empirical performance of the method has been confirmed on several benchmark datasets.
We note that our method can take in convolutional neural networks as feature extractions,
but additional efforts will be required to achieve high performance on image datasets.
This is one of important topics left for future work.

\section*{Acknowledgements}
This work was partially supported by MEXT KAKENHI (25730013, 25120012, 26280009, 15H01678 and 15H05707), JST-PRESTO (JPMJPR14E4), and JST-CREST (JPMJCR14D7, JPMJCR1304).

\bibliographystyle{icml2018}


\clearpage
\onecolumn
\renewcommand{\thesection}{\Alph{section}}
\renewcommand{\thesubsection}{\Alph{section}. \arabic{subsection}}
\renewcommand{\thetheorem}{\Alph{theorem}}
\renewcommand{\thelemma}{\Alph{lemma}}
\renewcommand{\theproposition}{\Alph{proposition}}
\renewcommand{\thedefinition}{\Alph{definition}}
\renewcommand{\thecorollary}{\Alph{corollary}}
\renewcommand{\theassumption}{\Alph{assumption}}

\setcounter{section}{0}
\setcounter{subsection}{0}
\setcounter{theorem}{0}
\setcounter{lemma}{0}
\setcounter{proposition}{0}
\setcounter{definition}{0}
\setcounter{corollary}{0}
\setcounter{assumption}{0}

\part*{\Large{Appendix}}

\section{Auxiliary Lemmas}
In this section, we introduce auxiliary lemmas used in our analysis.
The first one is Hoeffding's inequality.

\begin{lemma}[Hoeffding's inequality] \label{hoeffding_lemma}
  Let $Z_1,\ldots,Z_s$ be i.i.d. random variables to $[-a,a]$ for $a>0$.
  Denote by $A_s$ the sample average $\sum_{i=1}^sZ_i/s$.
  Then, for any $\epsilon > 0$, we get
  \begin{equation*}
    \prob[ A_s + \epsilon \leq \expec[A_s] ] \leq \exp\left( - \frac{\epsilon^2s}{2a^2}\right). 
  \end{equation*} 
\end{lemma}
Note that this statement can be reinterpreted as follows: it follows that for $\delta \in (0,1)$ with probability at least $1-\delta$
\begin{equation*}
 A_s + a\sqrt{\frac{2}{s}\log\frac{1}{\delta}} \geq \expec[A_s]. 
\end{equation*}

We next introduce the uniform bound by Rademacher complexity.
For a set $\mathcal{G}$ of functions from $\mathcal{Z}$ to $[-a,a]$ and a dataset $S=\{z_i\}_{i=1}^s \subset \mathcal{Z}$,
we denote empirical Rademacher complexity by $\empradcomp_S(\mathcal{G})$ and
denote Rademacher complexity by $\radcomp_s(\mathcal{G})$; let $\sigma =(\sigma_i)_{i=1}^s$ be
i.i.d random variables taking $-1$ or $1$ with equal probability and let $S$ be distributed according to a distribution $\mu^s$,
\[ \empradcomp_S(\mathcal{G}) = \expec_\sigma\left[ \sup_{f\in\mathcal{G}}\frac{1}{s}\sum_{i=1}^s\sigma_if(x_i)\right],
  \ \  \radcomp_s(\mathcal{G})=\expec_{\mu^s}[\empradcomp_S(\mathcal{G})]. \]

\begin{lemma} \label{rademacher_lemma}
  Let $Z_1,\ldots, Z_s$ be i.i.d random variables to $\mathcal{Z}$.
  Denote by $A_s(f)$ the sample average $\sum_{i=1}^sf(Z_i)/s$. 
  Then, for any $\delta \in (0,1)$, we get with probability at least $1-\delta$ over the choice of $S$,
  \begin{equation*}
    \sup_{f \in \mathcal{G}} \left| A_s(f) - \expec[A_s(f)] \right| \leq 2 \radcomp_s(\mathcal{G}) + a\sqrt{\frac{2}{s}\log\frac{2}{\delta}}.
  \end{equation*}
\end{lemma}

When a function class is VC-class (for the definite see \cite{vw1996}), its Rademacher complexity is uniformly bounded as in the following lemma which can be easily
shown by Dudley's integral bound \cite{dud1999} and the bound on the covering number by VC-dimension (pseudo-dimension) \cite{vw1996}. 

\begin{lemma} \label{uniform_bounded_complexity}
  Let $\mathcal{G}$ be VC-class.
  Then, there exists positive value $M$ depending on $\mathcal{G}$ such that $\radcomp_s(\mathcal{G}) \leq M/\sqrt{m}$.
\end{lemma}

The following lemma is useful in estimating Rademacher complexity.

\begin{lemma} \label{rad_comp_lemma}
  (i) Let $h_i : \realsp \rightarrow \realsp$ $(i\in \{1,\ldots,s\})$ be $L$-Lipschitz functions.
  Then it follows that
  \begin{equation*}
    \expec_\sigma\left[ \sup_{f\in\mathcal{G}} \sum_{i=1}^s\sigma_i h_i\circ f(x_i)\right]
    \leq L \expec_\sigma\left[ \sup_{f\in\mathcal{G}} \sum_{i=1}^s\sigma_i \circ f(x_i)\right].
  \end{equation*}
  (ii) We denote by ${\rm conv}(\mathcal{G})$ the convex hull of $\mathcal{G}$. Then, we have $\empradcomp_S({\rm conv}(\mathcal{G}))=\empradcomp_S(\mathcal{G})$.\\
\end{lemma}

The following lemma gives the generalization bound by the margin distribution, which is originally derived by \cite{koltchinskii2002empirical}.
Let $\mathcal{G}$ be the set of predictors; $\mathcal{G} \subset \{ f : \featuresp \rightarrow \realsp^\nclass \}$
and denote $\Pi \mathcal{G} = \{ f_y(\cdot) : \featuresp \rightarrow \mid f \in \mathcal{G}, y \in \labelsp \}$, then the following holds.

\begin{lemma} \label{margin_bound_lem}
  Fix $\delta > 0$.
  Then, for $\forall \rho > 0$, with probability at least $1-\rho$ over the random choice of $S$ from $\tpr^\ndata$, we have $\forall f \in \mathcal{G}$,
  \begin{equation*}
    \prob_{\tpr}[m_f(X,Y) \leq 0] \leq \prob_{\tpr_\ndata}[m_f(X,Y) \leq \delta]
    + \frac{2\nclass^2}{\delta} \radcomp_\ndata( \Pi \mathcal{G} ) + \sqrt{ \frac{1}{2n} \log \frac{1}{\rho}}.
  \end{equation*}
\end{lemma}

\section{Proofs}
In this section, we provide missing proofs in the paper.

\subsection{Proofs of Section  \ref{functional_gradient_prop_sec} and \ref{algorithm_description_sec}}
We first prove Proposition \ref{taylor_prop} that states Lipschitz smoothness of the risk function.

\begin{proof}[ Proof of Proposition \ref{taylor_prop} ]
  Because $l(z,y,w)$ is $\mathcal{C}^2$-function with respect to $z, w$, there exist semi-positive definite matrices $A_{x,y}^{\phi,\psi}, B_{x,y}^{\phi,\psi}$
  such that
  \begin{align}
    l(\psi(x),y,w_\phi) &= l(\phi(x),y,w_\phi) + \partial_z l(\phi(x),y,w_\phi)^\top(\psi(x)-\phi(x)) \notag \\
                        &+\frac{1}{2}(\psi(x)-\phi(x))^\top A_{x,y}^{\phi,\psi} (\psi(x)-\phi(x)), \label{taylor_expand_1} \\
    l(\psi(x),y,w_\phi) + \frac{\lambda}{2}\|w_\phi\|_2^2 &= l(\psi(x),y,w_\psi) + \frac{\lambda}{2}\|w_\psi\|_2^2 \notag \\
                        &+ (\partial_wl(\psi(x),y,w_\psi) + \lambda w_\psi)^\top (w_\phi-w_\psi) \notag \\
                        &+\frac{1}{2}(w_\phi-w_\psi)^\top B_{x,y}^{\phi,\psi} (w_\phi-w_\psi). \label{taylor_expand_2}
  \end{align}
  Note that we regard $w_\phi$ and $w_\psi$ are flattened into column vectors if necessary.
  By Assumption \ref{smoothness_assumption_1}, we find spectral norms of $A_{x,y}^{\phi,\psi}$ is uniformly bounded with respect to $x,y,\phi,\psi$,
  hence eigen-values are also uniformly bounded.
  In particular, since $\frac{\lambda}{2}\|w_\phi\|_2^2 \leq \regrisk(\phi,w_\phi) \leq \regrisk(\phi,0)\leq l_0$ , we see $- A_{c_\lambda}I \preceq A_{x,y}^{\phi,\psi} \preceq A_{c_\lambda}I$.

  By taking the expectation $\expec_\tpr$ of the equality (\ref{taylor_expand_1}), we get
  \begin{equation}
    \regrisk(\psi,w_\phi) = \regrisk(\phi,w_\phi) + \pd< \nabla_\phi \regrisk(\phi), \psi-\phi >_{L_2^{\fdim}(\tpr_X)} 
                        +\frac{1}{2} \expec_\tpr[ (\psi(x)-\phi(x))^\top A_{x,y}^{\phi,\psi} (\psi(x)-\phi(x)) ] \label{taylor_expand_1b}
  \end{equation}
  and by taking the expectation $\expec_\tpr$ of the equality (\ref{taylor_expand_2}), we get
  \begin{equation}
    \regrisk(\psi,w_\phi) = \regrisk(\psi,w_\psi) + \frac{1}{2}(w_\phi-w_\psi)^\top \expec_\tpr[B_{x,y}^{\phi,\psi}] (w_\phi-w_\psi),
    \label{taylor_expand_2b}
  \end{equation}
  where we used $\partial_w \regrisk(\psi,w_\psi)=0$.
  By combining equalities (\ref{taylor_expand_1b}) and (\ref{taylor_expand_2b}), we have
  \begin{equation*}
    \regrisk(\psi) = \regrisk(\phi) + \pd< \nabla_\phi \regrisk(\phi), \psi-\phi >_{L_2^{\fdim}(\tpr_X)} + H_\phi(\psi),
  \end{equation*}
  where
  \begin{equation*}
    H_\phi(\psi) = \frac{1}{2} \expec_\tpr[ (\psi(x)-\phi(x))^\top A_{x,y}^{\phi,\psi} (\psi(x)-\phi(x)) ]
    - \frac{1}{2}(w_\phi-w_\psi)^\top \expec_\tpr[B_{x,y}^{\phi,\psi}] (w_\phi-w_\psi).    
  \end{equation*}
  By the uniformly boundedness of $A_{x,y}^{\phi,\psi}$ and the semi-positivity of $B_{x,y}^{\phi,\psi}$,
  we find $H_\phi(\psi) \leq \frac{A_{c_\lambda}}{2}\|\phi-\psi\|_{L_2^{\fdim}(\tpr_X)}^2$.

  The other cases can be shown in the same manner, thus, we finish the proof.
\end{proof}

We next show the consistency of functional gradient norms.
\begin{proof}[ Proof of Proposition \ref{consistency_prop}]
  We now prove the first inequality.
  Note that the integrand of $y'$-th element of $\nabla_f \risk(f)(x)$ for multiclass logistic loss can be written as
  \begin{equation*}
    \partial_{\zeta_{y'}} l(f(x),y)
    = - {\bf 1}[y=y'] + \frac{\exp(f_{y'}(x))}{\sum_{\overline{y}\in \labelsp}\exp(f_{\overline{y}}(x)) }.
  \end{equation*}
  Therefore, we get
  
  \begin{align*}
    \| \nabla_f \risk(f) \|_{L_1^\nclass(\tpr_X)}
    &= \expec_{\tpr_X}\| \nabla_f \risk(f)(X) \|_2 \\
    &= \expec_{\tpr_X}\| \expec_{\tpr(Y|X)} [ \partial_\zeta (f(X),Y) ] \|_2 \\
    &= \expec_{\tpr_X}\left[ \sqrt{ \sum_{y' \in  \labelsp} (\expec_{\tpr(Y|X)} [ \partial_{\zeta_{y'}} (f(X),Y) ])^2 } \right] \\
    &\geq \frac{1}{\sqrt{\nclass}} \sum_{y' \in  \labelsp} \expec_{\tpr_X}\left[ \left|  \expec_{\tpr(Y|X)} [ \partial_{\zeta_{y'}} (f(X),Y) ]  \right| \right]\\
    &= \frac{1}{\sqrt{\nclass}} \sum_{y' \in  \labelsp}
      \expec_{\tpr_X}\left[ \left|
      \tpr(y'|X) \left(-1 + \frac{\exp(f_{y'}(X))}{\sum_{\overline{y}\in \labelsp} \exp(f_{\overline{y}}(X))} \right)
      + \sum_{y\neq y'}\tpr(y|X) \frac{\exp(f_{y'}(X))}{\sum_{\overline{y}\in \labelsp} \exp(f_{\overline{y}}(X))} 
      \right| \right]\\
    &= \frac{1}{\sqrt{\nclass}} \sum_{y' \in  \labelsp}
      \expec_{\tpr_X}\left[ \left|
      \tpr(y'|X) \left(-1 + \frac{\exp(f_{y'}(X))}{\sum_{\overline{y}\in \labelsp} \exp(f_{\overline{y}}(X))} \right)
      + (1-\tpr(y'|X)) \frac{\exp(f_{y'}(X))}{\sum_{\overline{y}\in \labelsp} \exp(f_{\overline{y}}(X))} 
      \right| \right]\\
    &= \frac{1}{\sqrt{\nclass}} \sum_{y' \in  \labelsp}
      \expec_{\tpr_X}\left[ \left|
      -\tpr(y'|X) 
      + \frac{\exp(f_{y'}(X))}{\sum_{\overline{y}\in \labelsp} \exp(f_{\overline{y}}(X))} 
      \right| \right]  \\
    &= \frac{1}{\sqrt{\nclass}} \sum_{y' \in  \labelsp}
      \| -\tpr(y'|\cdot) + p_{f}(y'|\cdot) \|_{L_1(\tpr_X)},
  \end{align*}
  where for the first inequality we used $(\sum_{i=1}^{\nclass}a_i)^2 \leq \nclass \sum_{i=1}^\nclass a_i^2$.
  Noting that the second inequality in Proposition \ref{consistency_prop} can be shown in the same way by replacing $\tpr$ by $\tpr_{\ndata}$,
  we finish the proof.
\end{proof}

We here give the proof of the following inequality concerning choice of embedding introduced
in section \ref{algorithm_description_sec}. 
\begin{equation}
  \|T_{k_t,\ndata}\partial_{\phi} \regrisk_\ndata(\phi_t,w_{t+1})\|_{k_t}^2
  \geq \frac{1}{\fdim}\|\partial_\phi \regrisk_\ndata(\phi_t,w_{t+1})\|_{L_1^\fdim(\tpr_{\ndata,X})}^2   \label{emb_ineq}
\end{equation}

\begin{proof}[ Proof of (\ref{emb_ineq}) ]
  For notational simplicity, we denote by $G_t=\partial_{\phi} \regrisk_\ndata(\phi_t,w_{t+1})(\cdot)$
  and by $G_t^i$ the $i$-the element of $G_t$.
  Then, we get
\begin{align*}
  \|T_{k_t,\ndata}(G_t)\|_2) \|_{k_t}^2
  &= \pd<G_t,T_{k_t,\ndata}G_t>_{L_2^\fdim(\tpr_{\ndata,X})}\\
  &= \expec_{(X,X') \sim \tpr_{\ndata,X}^2}[ G_t(X)^\top G_t(X')G_t(X')^\top G_t(X)/(\|G_t(X)\|_2\|G_t(X')\|_2) ] \\
  &= \sum_{i,j=1}^{\fdim}( \expec_{\tpr_{\ndata,X}}[ G_t^i(X) G_t^j(X)/\|G_t(X)\|_2] )^2 \\
  &\geq \sum_{i=1}^{\fdim}( \expec_{\tpr_{\ndata,X}}[G_t^i(X)^2/\|G_t(X)\|_2] )^2\\
  &\geq \frac{1}{\fdim}\expec_{\tpr_{\ndata,X}}[ \|G_t(X))\|_2 ]^2  
  = \frac{1}{\fdim}\| G_t \|_{L_1^{\fdim}(\tpr_{\ndata,X})}^2,
\end{align*}
where we used $(\sum_{i=1}^{\nclass}a_i)^2 \leq \nclass \sum_{i=1}^\nclass a_i^2$.
\end{proof}

\subsection{Empirical risk minimization and generalization bound}
In this section, we give the proof of convergence of Algorithm \ref{PFGD} for the empirical risk minimization.
We here briefly introduce the kernel function that provides useful bound in our analysis.
A kernel function $k$ is a symmetric function $\featuresp \times \featuresp \rightarrow \realsp$
such that for arbitrary $s \in \naturalsp$ and points $\forall (x_i)_{i=1}^s$, a matrix $(k(x_i,x_j))_{i,j=1}^s$ is positive semi-definite.
This kernel defines a reproducing kernel Hilbert space $\mathcal{H}_k$ of functions on $\featuresp$, which has two characteristic properties:
(i) for $\forall x \in \featuresp$, a function $k(x,\cdot):\featuresp \rightarrow \realsp$ is an element of $\mathcal{H}_k$,
(ii) for $\forall f \in \mathcal{H}_k$ and $\forall x \in \featuresp$, $f(x)=\pd<f,k(x,\cdot)>_{\mathcal{H}_k}$, where $\pd<,>_{\mathcal{H}_k}$ is the
inner-product in $\mathcal{H}_k$.
These properties are very important and the latter one is called reproducing property.
We extend the inner-product into the product space $\mathcal{H}_k^\fdim$ in a straightforward way, i.e.,
$\pd<f,g>_{\mathcal{H}_k^\fdim} = \sum_{i=1}^\fdim \pd<f^i,g^i>_{\mathcal{H}_k}$.

The following proposition is useful in our analysis.
The first property mean that the notation $\|T_{k_t,\ndata}\nabla \regrisk_\ndata(\phi_t)\|_{k_t}$ provided in the paper is nothing but the norm
of $T_{k_t,\ndata}\nabla \regrisk_\ndata(\phi_t)$ by the inner-product $\pd<,>_{\mathcal{H}_{k_t}^\fdim}$.

\begin{proposition} \label{kernel_property}
  For a kernel function $k$, the following hold.
  \begin{itemize}
  \item $\pd<f,g>_{L_2(\tpr_X)} = \pd<T_kf,g>_{\mathcal{H}_k^\fdim}$ for $f \in L_2^\fdim(\tpr_X),\ g \in \mathcal{H}_k^\fdim$ where $T_kf=\expec_{\tpr_X}[f(X)k(X,\cdot)]$,\\
  $\pd<f,g>_{L_2(\tpr_{\ndata,X})} = \pd<T_{k,\ndata}f,g>_{\mathcal{H}_k^\fdim}$ for $f \in L_2^\fdim(\tpr_{\ndata,X}),\ g \in \mathcal{H}_k^\fdim$
  where $T_{k,\ndata}f=\expec_{\tpr_{\ndata,X}}[f(X)k(X,\cdot)]$,
\item $\|f\|_{L_2(\tpr_X)}^2 \leq \expec_{\tpr_X}[k(X,X)]\|f\|_{\mathcal{H}_k^\fdim}^2$ for $f \in \mathcal{H}_k^\fdim$, \\
  $\|f\|_{L_2(\tpr_{\ndata,X})}^2 \leq \expec_{\tpr_{\ndata,X}}[k(X,X)]\|f\|_{\mathcal{H}_k^\fdim}^2$ for $f \in \mathcal{H}_k^\fdim$.  
  \end{itemize}
  
\end{proposition}
\begin{proof}
  We show only the case of $\tpr_X$ because we can prove the other case in the same manner.
  For $f \in L_2(\tpr_X), g \in \mathcal{H}_k^\fdim$, we get the first property by using reproducing property,
  \begin{align*}
    \pd<f,g>_{L_2(\tpr_X)}= \expec_{\tpr_X}[ f(X)\top \pd<g,k(X,\cdot)>_{\mathcal{H}_k^\fdim}]
    = \pd<g,T_kf>_{\mathcal{H}_k^\fdim}.
  \end{align*}
  
  We next show the second property as follows.
  For $\forall f \in \mathcal{H}_k^\fdim$, we get
  \begin{align*}
    \|f\|_{L_2(\tpr_X)}^2
    &= \expec_{\tpr_X}\|f(X)\|_2^2 \\
    &= \expec_{\tpr_X} \| \pd<f(\cdot),k(X,\cdot)>_{\mathcal{H}_k^\fdim}\|_2^2\\
    &\leq \expec_{\tpr_X}\|k(X,\cdot)\|_{\mathcal{H}_k}^2 \|f\|_{\mathcal{H}_k^\fdim}^2 \\
    &= \expec_{\tpr_X}[k(X,X)] \|f\|_{\mathcal{H}_k^\fdim}^2. 
  \end{align*}

\end{proof}

We give the proof of Theorem \ref{convergence_thm} concerning the convergence of functional gradient norms.

\begin{proof} [ Proof of Theorem \ref{convergence_thm} ]
  When $\eta \leq \frac{1}{A_{c_\lambda}K}$, we have from Proposition \ref{taylor_prop} and Proposition \ref{kernel_property},
\begin{equation*}
  \regrisk_\ndata(\phi_{t+1},w_{t+2}) \leq \regrisk_\ndata(\phi_t,w_{t+1}) 
   - \frac{\eta}{2}\|T_{k_t,\ndata}\partial_{\phi} \regrisk_\ndata(\phi_t,w_{t+1})\|_{k_t}^2.  
 \end{equation*}
 By Summing this inequality over $t\in \{0,\ldots,T-1\}$ and dividing by $T$, we get
 \begin{equation}
   \frac{1}{T}\sum_{t=0}^{T-1}\|T_{k_t,\ndata}\partial_{\phi} \regrisk_\ndata(\phi_t,w_{t+1})\|_{k_t}^2
   \leq \frac{2}{\eta T}\regrisk_\ndata(\phi_0,w_{1}), \label{conv_thm_ineq_1}
 \end{equation}
 where we used $\regrisk_\ndata \geq 0$.

 On the other hand, since $\partial_z l(z,y,w) =\partial_z l(\predict{w}{z},y) = w \partial_\zeta l(\predict{w}{z},y)$,
 it follows that
 \begin{align*}
   \partial_\phi \regrisk_\ndata( \phi, w )(x)
   &= \expec_{\tpr_\ndata(Y|x)}[ \partial_z l( \phi(x), y, w) ] \\
   &= \expec_{\tpr_\ndata(Y|x)}[ w \partial_\zeta l( \predict{w}{\phi(x)}, y) ] \\
   &= w \nabla_f \risk_\ndata( \predict{w}{\phi} )(x).
 \end{align*}

 Thus, by the assumption on $(w_t{^\top}w_t)_{t=0}^{T_0}$, we get for $t \in \{0,\ldots,T-1\}$
 \begin{align}
   \| \partial_\phi \regrisk_\ndata( \phi_t, w_{t+1} ) \|_{L_p^{\fdim}(\tpr_{\ndata,X})}
   &= \expec_{\tpr_{\ndata,X}}[ \| w_{t+1} \nabla_f \risk_\ndata( \predict{w_{t+1}}{\phi_t} )(X) \|_2^{p} ]^{1/p} \notag \\
   &\geq \sigma \expec_{\tpr_{\ndata,X}}[ \| \nabla_f \risk_\ndata( \predict{w_{t+1}}{\phi_t} )(X) \|_2^p ]^{1/p} \notag \\
   &= \sigma \| \nabla_f \risk_\ndata( \predict{w_{t+1}}{\phi_t}) \|_{L_p^{\nclass}(\tpr_{\ndata,X})}. \label{conv_thm_ineq_2}
 \end{align}
 
Combining inequalities (\ref{conv_thm_ineq_1}) (\ref{conv_thm_ineq_2}) and Assumption \ref{kernel_choice_assumption}, we get

 \begin{equation*}
\frac{1}{T}\sum_{t=0}^{T-1}   
   \| \nabla_f \risk_\ndata( \predict{w_{t+1}}{\phi_t}) \|_{L_p^{\nclass}(\tpr_{\ndata,X})}^q 
   \leq \frac{2}{\eta \gamma \sigma^q T}\regrisk_\ndata(\phi_0,w_{1}) + \frac{\epsilon}{\sigma^q}.
 \end{equation*}
 Since $p\geq 1$, we observe $\| \nabla_f \risk_\ndata( \predict{w_{t+1}}{\phi_t}) \|_{L_1^{\nclass}(\tpr_{\ndata,X})} \leq \| \nabla_f \risk_\ndata( \predict{w_{t+1}}{\phi_t}) \|_{L_p^{\nclass}(\tpr_{\ndata,X})}$
 and we finish the proof.
\end{proof}

To provide the proof of Theorem \ref{fast_convergence_thm}, we here give an useful proposition to show fast convergence rate for the multiclass logistic regression.
\begin{proposition} \label{gradient_objective_gap_bound_prop}
  Let $l(\zeta,y)$ be the loss function for the multiclass logistic regression.
  Let $M>0$ be arbitrary constant.
  For a predictor $f$, we assume $l(f(X),Y)\leq M$ for $(X,Y)\sim \tpr_{\ndata}$.
  Then, we have
  \[ \| \nabla_f \risk_\ndata(f)\|_{L_1^\nclass(\tpr_{\ndata,X})} \geq \frac{1-\exp(-M)}{\sqrt{c}M} \risk_{\ndata}(f). \]
\end{proposition}
\begin{proof}
  Since $\exp(-t) \leq 1-\frac{1-\exp(-M)}{M}t$ for $\forall t \in [-M,0)$, we get
  \[ \expec_{\tpr_{\ndata}}[ \exp( -l(f(X),Y) ) ] \leq 1 - \frac{1-\exp(-M)}{M}\risk_\ndata(f). \]
  Using $l(f(X),Y) = -\log p_f(Y|X)$ and the above inequality with Proposition \ref{consistency_prop}, we obtain
  \begin{align*}
    \| \nabla_f \risk_\ndata(f) \|_{L_1^\nclass(\tpr_{\ndata,X})}
    &\geq \frac{1}{\sqrt{\nclass}}\sum_{y\in \labelsp} \| \tpr_{\ndata}(y|\cdot) - p_f(y|\cdot) \|_{L_1(\tpr_{\ndata,X})} \\
    &= \frac{1}{\sqrt{\nclass}}\sum_{y\in \labelsp} \expec_{\tpr_{\ndata,X}}| \tpr_{\ndata}(y|X) - p_f(y|X) | \\
    &\geq \frac{1}{\sqrt{\nclass}} \expec_{\tpr_{\ndata}}[ 1 - p_f(Y|X) ] \\
    &= \frac{1}{\sqrt{\nclass}} \expec_{\tpr_{\ndata}}[ 1 - \exp( -l(f(X),Y) ) ] \\
    &\geq \frac{1-\exp(-M)}{\sqrt{c}M}\risk_\ndata(f).
  \end{align*}
\end{proof}

The following is the proof for Theorem \ref{fast_convergence_thm}.
\begin{proof} [ Proof of Theorem \ref{fast_convergence_thm} ]
  Noting that $f_{t+1} \leftarrow f_t - \eta w_{t+1}^\top T_{k_t,\ndata}\partial_{\phi} \regrisk_\ndata(\phi_t,w_{t+1})$,
  we get the following bound by a similar way in the proof for Theorem \ref{convergence_thm}.
  \begin{equation*} 
    \risk_\ndata( f_{t+1}) \leq \risk(f_t) - \eta \pd< \nabla \risk_\ndata(f_t), w_{t+1}^\top T_{k_t,\ndata}\partial_{\phi} \regrisk_\ndata(\phi_t,w_{t+1})>_{L_2^{\nclass}(\tpr_{\ndata,X})}
    + \frac{A\eta^2}{2}\|w_{t+1}^\top T_{k_t,\ndata}\partial_{\phi} \regrisk_\ndata(\phi_t,w_{t+1})\|_{L_2^{\nclass}(\tpr_{\ndata,X})}^2.
  \end{equation*}
  We here bound the right hand side of this inequality as follows.
  \begin{align*}
    \pd< \nabla \risk_\ndata(f_t), w_{t+1}^\top T_{k_t,\ndata}\partial_{\phi} \regrisk_\ndata(\phi_t,w_{t+1})>_{L_2^{\nclass}(\tpr_{\ndata,X})} \notag
    & = \pd< w_{t+1}\nabla \risk_\ndata(f_t), T_{k_t,\ndata}\partial_{\phi} \regrisk_\ndata(\phi_t,w_{t+1})>_{L_2^{\fdim}(\tpr_{\ndata,X})} \notag \\
    & = \pd< \partial_{\phi} \regrisk_\ndata(\phi_t,w_{t+1}), T_{k_t,\ndata}\partial_{\phi} \regrisk_\ndata(\phi_t,w_{t+1})>_{L_2^{\fdim}(\tpr_{\ndata,X})} \notag \\
    & \geq \gamma \|\partial_{\phi} \regrisk_\ndata(\phi_t,w_{t+1})\|_{L_{1}^{\fdim}(\tpr_{\ndata,X})}^2, 
  \end{align*}
  where we used Proposition \ref{kernel_property} for the second equality and Assumption \ref{kernel_choice_assumption} for the last inequality.
  Recalling $\|w_{t+1}\|_2 \leq c_\lambda$, we have
  \begin{align*}
    \| w_{t+1}^\top T_{k_t,\ndata}\partial_{\phi} \regrisk_\ndata(\phi_t,w_{t+1}) \|_{L_2^{\nclass}(\tpr_{\ndata,X})}^2
    &= \expec_{X\sim \tpr_{\ndata,X}}\| w_{t+1}^\top T_{k_t,\ndata}\partial_{\phi} \regrisk_\ndata(\phi_t,w_{t+1}) \|_2^2 \notag \\
    &\leq c_\lambda^2 \expec_{\tpr_{\ndata,X}}\| T_{k_t,\ndata}\partial_{\phi} \regrisk_\ndata(\phi_t,w_{t+1})(X) \|_2^2 \notag \\
    &\leq c_\lambda^2 K^2 \| \partial_{\phi} \regrisk_\ndata(\phi_t,w_{t+1})(X) \|_{L_{1}^{\fdim}(\tpr_{\ndata,X})}^2. 
  \end{align*}
  where for the second inequality, we used $\| T_{k_t,\ndata}\partial_{\phi} \regrisk_\ndata(\phi_t,w_{t+1})(X) \|_2 \leq K\| \partial_{\phi} \regrisk_{\ndata}(\phi_t,w_{t+1})\|_{L_1^\fdim(\tpr_{\ndata,X})}$
  which is a consequence of the triangle inequality.
  Combining the above three inequalities, we have 
  \begin{align}
    \risk_\ndata( f_{t+1})
    &\leq \risk_\ndata(f_t) - \eta \left(\gamma - \frac{1}{2}A\eta c_\lambda^2 K^2\right) \|\partial_{\phi} \regrisk_\ndata(\phi_t,w_{t+1})\|_{L_{1}^{\fdim}(\tpr_{\ndata,X})}^2 \notag \\
    &\leq \risk_\ndata(f_t) - \frac{\eta\gamma}{2} \|\partial_{\phi} \regrisk_\ndata(\phi_t,w_{t+1})\|_{L_{1}^{\fdim}(\tpr_{\ndata,X})}^2 \notag \\
    &\leq \risk_\ndata(f_t) - \frac{\eta\gamma \sigma^2}{2} \|\nabla_f\risk_{\ndata}(f_t) \|_{L_1^\nclass(\tpr_{\ndata,X})}^2, \label{fast_convergence_ineq_1}
  \end{align}
  where we used $A\eta c_\lambda^2K^2 \leq \gamma$ for the second inequality and we used (\ref{conv_thm_ineq_2}) for the last inequality.
  Thus, we obtain from (\ref{fast_convergence_ineq_1}) and Proposition \ref{gradient_objective_gap_bound_prop} that
  \begin{align*}
    \risk_\ndata( f_{t+1})
    \leq \risk_\ndata(f_t) - \eta\alpha \risk_{\ndata}^2(f_t),
  \end{align*}
  where $\alpha = \frac{\gamma \sigma^2(1-\exp(-M))^2}{2cM^2}$.

  From this inequality, we get
  \[ \frac{ 1 }{ \risk_\ndata( f_{t}) }
    \geq \frac{ 1 }{ \risk_\ndata( f_{t+1}) } - \eta\alpha \frac{ \risk_\ndata( f_{t})}{ \risk_\ndata( f_{t+1})}
    \geq \frac{ 1 }{ \risk_\ndata( f_{t+1}) } - \eta\alpha, \]
  where for the last inequality we used the fact that $\risk_\ndata(f_t)$ is monotone decreasing which is confirmed from the inequality (\ref{fast_convergence_ineq_1}).
  Therefore, by applying this bound recursively for $t \in \{ 0,\ldots,T/2 - 1 \}$ with $\eta= \eta_0$, we conclude
  \begin{equation} 
    \risk_\ndata( f_{T/2} ) \leq \frac{2\risk_\ndata(f_0)}{2+\eta_0\alpha \risk_\ndata(f_0)T}. \label{convex_convergence_rate}
  \end{equation}

  On the other hand, by summing up the inequality (\ref{fast_convergence_ineq_1}) over $t \in \{T/2,\ldots,T-1\}$ with $\eta_1$ and dividing by $T/2$, we get
  \begin{equation}
    \frac{\eta_1\gamma \sigma^2}{T}\sum_{t=T/2}^{T-1} \|\nabla_f\risk_{\ndata}(f_t) \|_{L_1^\nclass(\tpr_{\ndata,X})}^2 \leq \frac{2}{T}\risk_\ndata(f_{T/2}).  \label{nonconvex_convergence_rate}
  \end{equation}

  From inequalities (\ref{convex_convergence_rate}) and (\ref{nonconvex_convergence_rate}), we conclude
  \[ \frac{\eta_1\gamma \sigma^2}{T}\sum_{t=T/2}^{T-1} \|\nabla_f\risk_{\ndata}(f_t) \|_{L_1^\nclass(\tpr_{\ndata,X})}^2
  \leq \frac{4\risk_\ndata(f_0)}{(2+\eta_0\alpha \risk_\ndata(f_0)T)T}. \]  
\end{proof}

We next show Theorem \ref{margin_bound_thm} that gives the generalization bound by the margin distribution.
To do that, we give an upper-bound on the margin distribution by the functional gradient norm.

\begin{proposition} \label{margin_dist_bound}
  For $\forall \delta > 0$, the following bound holds.
  \[ \prob_{\tpr_\ndata}[ m_f(X,Y) \leq \delta] \leq \left( 1 + \frac{1}{\exp(-\delta)}\right) \sqrt{\nclass} \| \nabla_f \risk_\ndata(f) \|_{L_1^\nclass(\tpr_{\ndata,X})}. \]
\end{proposition}
\begin{proof}
  If $m_f(x,y)\leq \delta$, then, we see
  \begin{equation*}
    \sum_{y'\neq y}\exp( f_{y'}(x)-f_y(x)) \geq  \exp\left( \max_{y' \neq y} f_{y'}(x)-f_y(x) \right) = \exp( -m_f(x,y)) \geq \exp(-\delta).
  \end{equation*}

  This implies,
  \begin{equation*}
    p_f(y|x) = \frac{1}{ 1 + \sum_{y'\neq y}\exp( f_{y'}(x)-f_y(x))} \leq \frac{1}{ 1 + \exp(-\delta)}.
  \end{equation*}  
  
  Thus, we get by Markov inequality and Proposition \ref{consistency_prop},
  \begin{align*}
    \prob_{\tpr_\ndata}[ m_f(X,Y) \leq \delta]
    &\leq \prob_{\tpr_\ndata}\left[ p_f(Y|X) \leq \frac{1}{1+\exp(-\delta)} \right]  \\
    &= \prob_{\tpr_\ndata}\left[ 1-p_f(Y|X) \geq \frac{ \exp(-\delta)}{1+\exp(-\delta)} \right] \\
    &\leq \left( 1 + \frac{1}{\exp(-\delta)}\right) \expec_{\tpr_\ndata}[ 1- p_f(Y|X)] \\
    &= \left( 1 + \frac{1}{\exp(-\delta)}\right) \expec_{\tpr_\ndata}[ \nu_n(Y|X)- p_f(Y|X)] \\
    &\leq \left( 1 + \frac{1}{\exp(-\delta)}\right) \sum_{y \in \labelsp} \| \nu_n(y|\cdot)- p_f(y|\cdot)] \|_{L_1(\tpr_{n,X})} \\    
    &\leq \left( 1 + \frac{1}{\exp(-\delta)}\right) \sqrt{\nclass} \| \nabla_f \risk_\ndata(f) \|_{L_1^\nclass(\tpr_{\ndata,X})}.
  \end{align*}
\end{proof}

We prove here Theorem \ref{margin_bound_thm}.
\begin{proof}[ Proof of Theorem \ref{margin_bound_thm} ]
  To proof the theorem, we give the network structure.
  Note that the connection at the $t$-th layer is as follows.
  \[ \phi_{t+1}(x) = \phi_{t}(x) - D_t \sigma(C_t \phi_t(x)). \]
  
  We define recursively the family of functions $\mathcal{H}_t$ and $\hat{\mathcal{H}}_t$ where each neuron belong:
  We denote by $P_j \in \realsp^\fdim$ the projection vector to $j$-th coordinate.
  \begin{align*}
    \mathcal{H}_0 &\defeq \{ P_j : \featuresp \rightarrow \realsp \mid j \in \{1,\ldots,\fdim\} \}, \\
    \hat{\mathcal{H}}_{t} &\defeq \{ \sigma( c_{t}^\top \phi_t ) : \featuresp \rightarrow \realsp \mid \phi_t \in \mathcal{H}_t^\fdim, c_{t-1} \in \realsp^\fdim, \|c_{t-1}\|_1 \leq \Lambda \}, \\
    \mathcal{H}_{t+1} &\defeq \{ \phi_{t}^j - d_{t}^\top \psi_{t} : \featuresp \rightarrow \realsp
                      \mid \phi_{t}^j \in \mathcal{H}_{t}, \psi_{t} \in \hat{\mathcal{H}}_{t}^\fdim, d_{t} \in \realsp^\fdim, \|d_{t}\|_1 \leq \Lambda'_t \}. 
    \end{align*}
    Then, the family of predictors of $y \in \labelsp$ can be written as
    \[ \mathcal{G}_{T-1,y} \defeq \{ w_y^\top \phi_{T-1} : \featuresp \rightarrow \realsp \mid \phi \in \mathcal{H}_{T-1}^\fdim, w_y \in \realsp^\fdim, \|w_y\|_1 \leq \Lambda_w \}. \]
    Note that $\mathcal{G}_{T-1}= \{ (f_y)_{y\in \labelsp} \mid f_y \in \mathcal{G}_{T-1,y}, y \in \labelsp \}$.
    
    From these relationships and Lemma \ref{rad_comp_lemma}, we get
    \begin{align*}
      \empradcomp_S( \mathcal{H}_{t} ) &\leq \empradcomp_S( \mathcal{H}_{t-1} ) + \Lambda'_{t-1} \empradcomp_S( \hat{\mathcal{H}}_{t-1} ) \\
                                       &\leq ( 1  + \Lambda'_{t-1} \Lambda L_\sigma ) \empradcomp_S( \mathcal{H}_{t-1} ),\\
      \empradcomp_S( \mathcal{G}_{T-1,y} ) &\leq \Lambda_w \empradcomp_S( \mathcal{H}_{T-1} ).
    \end{align*}
    The Rademacher complexity of $\mathcal{H}_0$ is obtained as follows.
    Since $\|P_j\|_2=1$, we have
    \begin{align*}
      \empradcomp_S( \mathcal{H}_{0} )
      &= \frac{1}{\ndata} \expec_{(\sigma_i)_{i=1}^\ndata}\left[ \sup_{j\in \{1,\ldots,\fdim\}}\sum_{i=1}^\ndata \sigma_i P_jx_i \right] \\
      &\leq \frac{1}{\ndata} \expec_{(\sigma_i)_{i=1}^\ndata}\left[ \sup_{j\in \{1,\ldots,\fdim\}} \|P_j\|_2 \left\| \sum_{i=1}^\ndata \sigma_i x_i \right\|_2 \right] \\
      &= \frac{1}{\ndata} \expec_{(\sigma_i)_{i=1}^\ndata}\left[ \left\| \sum_{i=1}^\ndata \sigma_i x_i \right\|_2 \right] \\
      &\leq \frac{1}{\ndata}  \left( \expec_{(\sigma_i)_{i=1}^\ndata} \left[ \left\| \sum_{i=1}^\ndata \sigma_i x_i \right\|_2^2 \right] \right)^{\frac{1}{2}} \\
      &= \frac{1}{\ndata} \left( \sum_{i=1}^\ndata \| x_i \|_2^2 \right)^{\frac{1}{2}} 
      \leq \frac{\Lambda_{\infty}}{\sqrt{\ndata}},
    \end{align*}
    where we used the independence of $\sigma_i$ when taking the expectation.

    We set $\Pi \mathcal{G}_{T-1} = \{ f_y(\cdot) : \featuresp \rightarrow \mid f \in \mathcal{G}_{T-1}, y \in \labelsp\}$.
    Noting that $\empradcomp_S( \Pi \mathcal{G}_{T-1} )\leq \sum_{y \in \labelsp}\empradcomp_S(\mathcal{G}_{T-1,y} )$, we get
    \[ \empradcomp_S( \Pi \mathcal{G}_{T-1} ) \leq \nclass \Lambda_w\Lambda_{\infty}\prod_{t=0}^{T-2}(1+\Lambda\Lambda'_tL_\sigma)/\sqrt{n}. \]
    Thus, we can finish the proof by applying Proposition \ref{margin_dist_bound} and Lemma \ref{margin_bound_lem}.
  \end{proof}

Corollary \ref{union_bound} can be derived by instantiating Theorem \ref{margin_bound_thm} for various choices of $T,\ \Lambda_t'$.
\begin{proof}[ Proof of Corollary \ref{union_bound} ]
  For simplicity, we set $v_{t,j}$ the $L_1$-norm of $j$-th row of $D_t$, namely, $v_{t,j} \defeq \| (D_t)_{j,*}\|_{1}$.
  For arbitrary positive integers $(T,\overline{k}_{T-2})=(T,k_0,\ldots,k_{T-2})$, we set $B(T,\overline{k}_{T-2})$ to networks defined by parameters included in 
  \[ \left\{ \max_j v_{t,j} \leq \frac{k_t}{T},\ \max_c \|(w_T)_{*,c}\|_{1} \leq \Lambda_w,\ \max_i \|(C_t)_{i,*}\|_{1}\leq \Lambda,\ \forall t \in \{0,\ldots,T-2\} \right\} \]
  and set
  \[ \rho(T,\overline{k}_{T-2}) \defeq \frac{\rho}{T(T+1)k_0(k_0 + 1)\cdots k_{T-2}(k_{T-2}+1)}. \]
  Moreover, we set $B\defeq \cup_{T,\overline{k}_{T-2}} B(T,\overline{k}_{T-2})$.
  Clearly, we see $\sum_{T,\overline{k}_{T-2}} \rho(T,\overline{k}_{T-2}) = \rho$. 
  Therefore, by instantiating Theorem \ref{margin_bound_thm} for all $(T,\overline{k}_{T-2})$ with probability at least $1-\rho(T,\overline{k}_{T-2})$ and taking an union bound,
  we have that with probability at least $1-\rho$ for $\forall f \in B$,
  \begin{align*}
    \prob_{\tpr}[m_f(X,Y) \leq 0]
    &\leq \frac{2\nclass^3\Lambda_{\infty}\Lambda_w}{\delta\sqrt{\ndata}}\prod_{t=0}^{T-2}\left(1+\Lambda L_\sigma\frac{k_t}{T}\right)
    + \sqrt{\frac{1}{2\ndata}\log\left(\frac{1}{\rho}\right) + 2\log(T+1) + \sum_{t=0}^{T-2}2\log(k_t+1) } \\
    &+ \left(1+\frac{1}{\exp(-\delta)}\right) \sqrt{\nclass} \| \nabla_f \risk_\ndata(f) \|_{L_1^\nclass(\tpr_{\ndata,X})}.
  \end{align*}
  Let $f \in B$ be a function obtained by Algorithm \ref{PFGD} and $(f_t)=(w_{t+1}^\top \phi_t)$ be a sequence to obtain $f$ in the algorithm.
  We choose the minimum integers $(T,\overline{k}_{T-2})$ such that $f \in B(T,\overline{k}_{T-2})$, then
  \[ \max_{j}v_{t,j} \leq \frac{k_t}{T} \leq \max_{j}v_{t,j} + \frac{1}{T}. \]
  Thus, it follows that
  \begin{align*}
      \prob_{\tpr}[m_f(X,Y) \leq 0]
    &\leq \left(1+\frac{1}{\exp(-\delta)}\right) \sqrt{\nclass} \| \nabla_f \risk_\ndata(f) \|_{L_1^\nclass(\tpr_{\ndata,X})} +
      \frac{2\nclass^3\Lambda_{\infty}\Lambda_w}{\delta\sqrt{\ndata}}\prod_{t=0}^{T-2}\left(1+\Lambda L_\sigma \left(\max_{j}v_{t,j} + \frac{1}{T}\right) \right) \\
    &+ \sqrt{\frac{1}{2\ndata}\left(\log\left(\frac{1}{\rho}\right) + 2\log(T+1) + \sum_{t=0}^{T-2}2\log(T\max_{j}v_{t,j} + 2)\right) }.
  \end{align*}
  We next estimate an upper bound on $\max_j v_{t,j}$.
  Note that $\|(A_tB_t)_{i,*}\|_{1} \leq \|(A_t)_{i,*}\|_2 \sum_{l}\|(B_t)_{*,l}\|_2 \leq \|(A_t)_{i,*}\|_2\Lambda''$
  and $\sum_i \|(A_t)_{i,*}\|_2 \leq \sqrt{K\fdim}\|\partial_\phi\regrisk_{\ndata}(\phi_t,w_{t+1})\|_{L_1^\fdim(\tpr_{\ndata,X})}$ by its construction.
  Thus, we get
  \[ \max_j v_{t,j} \leq \sum_{j}v_{t,j}
    \leq \eta_t \sqrt{K\fdim} \Lambda'' \| \partial_\phi \regrisk_{\ndata}(\phi_t,w_{t+1})\|_{L_1^\fdim(\tpr_{\ndata,X})}
    \leq \eta_t \sqrt{K\fdim} \Lambda'' c_\lambda \| \nabla_f \risk_\ndata(f_t) \|_{L_1^\nclass(\tpr_{\ndata,X})}, \]
  where we used $\|w_t\|_2 \leq c_\lambda$ for the last inequality.
  
  Using the inequality $\prod_{t=0}^{T-2}(1+p_t) \leq \left( 1 + \frac{1}{T-1}\sum_{t=0}^{T-2}p_t \right)^{T-1}$ for positive integers $(p_t)$ and Jensen's inequality,
  we have
  \begin{align*}
      \prob_{\tpr}[m_f(X,Y) \leq 0]
    &\leq \left(1+\frac{1}{\exp(-\delta)}\right) \sqrt{\nclass} \| \nabla_f \risk_\ndata(f) \|_{L_1^\nclass(\tpr_{\ndata,X})}
      + \frac{\nclass^3\Lambda_{\infty}\Lambda_w}{\delta\sqrt{\ndata}}\left(1+ \frac{2\Lambda L_\sigma}{T} \right)^{T-1} \\
    &+ \frac{\nclass^3\Lambda_{\infty}\Lambda_w}{\delta\sqrt{\ndata}}\left(1+\frac{2\Lambda L_\sigma \sqrt{K\fdim} \Lambda'' c_\lambda}{T-1}\sum_{t=0}^{T-2}  \eta_t \| \nabla_f \risk_\ndata(f_t) \|_{L_1^\nclass(\tpr_{\ndata,X})}  \right)^{T-1} \\    
    &+ \sqrt{\frac{1}{2\ndata}\left(\log\left(\frac{1}{\rho}\right) + O(T\log T)\right) }.
  \end{align*}

  Since $\left(1+ \frac{2\Lambda L_\sigma}{T} \right)^{T-1}$ is an increasing with respect to $T$ and converges to $2\Lambda L_\sigma$, we finish the proof.
\end{proof}

\subsection{Sample-splitting technique}
In this subsection, we provide proofs for the convergence analysis of the sample-splitting variant of the method for the expected risk minimization.
We first give the statistical error bound on the gap between the empirical and expected functional gradients.

\begin{proof} [ Proof of Proposition \ref{stat_error_prop} ]
  For the probability measure $\tpr$, we denote by $\phi_\sharp\tpr$ the push-forward measure $(\phi,id)_\sharp \tpr$, namely,
  $(\phi,id)_\sharp \tpr$ is the measure that the random variable $(\phi(X),Y)$ follows.
  We also define $\phi_\sharp\tpr_m$ in the same manner.
  Then, we get
  \begin{align}
    &\hspace{-10mm} \left\| T_{k}\partial_\phi \regrisk(\phi,w_{0}) - T_{k,m}\partial_\phi \regrisk_{m}(\phi,w_{0}) \right\|_{L_2^\fdim(\mu)} \notag \\
            &= \sqrt{ \expec_{X'\sim \mu}\|
              \expec_{\tpr}[\partial_z l(\phi(X),Y,w_0)k(X,X')]
              - \expec_{\tpr_m}[\partial_z l(\phi(X),Y,w_0)k(X,X')] 
              \|_2^2 }. \notag \\
            &= \sqrt{ \sum_{j=1}^\fdim \expec_{X'\sim \mu} | 
              ( \expec_{\tpr}[ \partial_{z_j} l(\phi(X),Y,w_0) \iota(\phi(X)) )]
              - \expec_{\tpr_m}[\partial_{z_j} l(\phi(X),Y,w_0)\iota(\phi(X)) ] )^\top \iota(\phi(X')) |^2 }. \notag\\
            &\leq \sqrt{ K \sum_{j=1}^\fdim 
              \| \expec_{\tpr}[ \partial_{z_j} l(\phi(X),Y,w_0) \iota(\phi(X)) )]
              - \expec_{\tpr_m}[\partial_{z_j} l(\phi(X),Y,w_0)\iota(\phi(X)) ] \|_2^2  } \notag \\
            &\leq \sqrt{ K \sum_{j=1}^\fdim \sum_{i=1}^\edim 
              \left| \expec_{\phi_{\sharp}\tpr}[ \partial_{z_j} l(X,Y,w_0) \iota^i(X) )]
              - \expec_{\phi_{\sharp}\tpr_m}[\partial_{z_j} l(X,Y,w_0)\iota^i(X) ] \right|^2  }. \label{stat_error_bound_1}
  \end{align}
  
  To derive an uniform bound on (\ref{stat_error_bound_1}), we estimate Rademacher complexity of
  \[ \mathcal{G}_{ij} \defeq \{ \partial_{z_j}l(x,y,w_0) \iota^i(x) : \featuresp \times \labelsp \rightarrow \realsp \mid \iota^i \in \mathcal{F}^i\}. \]
  For $(x_l,y_l)_{l=1}^m \subset \featuresp \times \labelsp$, we set $h_l(r) = r\partial_{z_j}l(x_l,y_l,w_0)$.
  Since, $|\partial_{z_j}l(x_l,y_l,w_0)| \leq \beta_{\|w_0\|_2}$ by Assumption \ref{smoothness_assumption_2}, $h_l$ is $\beta_{\|w_0\|_2}$-Lipschitz continuous.
  Thus, from Lemma \ref{uniform_bounded_complexity} and Lemma \ref{rad_comp_lemma}, there exists $M$ such that for all $i \in \{1,\ldots,\edim\},\ j\in \{1,\ldots,\fdim\}$,

  \begin{align*}
    \empradcomp_m(\mathcal{G}_{ij})
    &= \expec_\sigma\left[ \sup_{\iota^i \in\mathcal{F}^i} \sum_{l=1}^m\sigma_l h_l( \iota^i(x_l) ) \right] \\
    &\leq \beta_{\|w_0\|_2} \expec_\sigma\left[ \sup_{\iota^i \in\mathcal{F}^i} \sum_{l=1}^m\sigma_l \iota^i(x_l) \right] \\
    &\leq \beta_{\|w_0\|_2} \frac{M}{\sqrt{m}}.
  \end{align*}

  Therefore, by applying Lemma \ref{rademacher_lemma} with $\delta = \frac{\rho}{\fdim\edim}$ for $\forall i, j$ simultaneously,
  it follows that with probability at least $1-\rho$ for $\forall i, j$

  \begin{equation}
    \sup_{\iota^i \in \mathcal{F}^i}\left| \expec_{\phi_{\sharp}\tpr}[ \partial_{z_j} l(X,Y,w_0) \iota^i(X) )]
      - \expec_{\phi_{\sharp}\tpr_m}[\partial_{z_j} l(X,Y,w_0)\iota^i(X) ] \right|
    \leq \frac{\beta_{\|w_0\|_2}}{\sqrt{m}} \left( 2 M + \sqrt{ 2K \log\frac{2\fdim\edim}{\rho} } \right). \label{stat_error_concent_ineq}
  \end{equation}

  Putting (\ref{stat_error_concent_ineq}) int (\ref{stat_error_bound_1}), we get with probability at least $1-\rho$
  \begin{equation*}
    \sup_{\iota \in \mathcal{F}} \left\| T_{k}\partial_\phi \regrisk(\phi,w_{0}) - T_{k,m}\partial_\phi \regrisk_{m}(\phi,w_{0}) \right\|_{L_2^\fdim(\mu)}
    \leq \beta_{\|w_0\|_2} \sqrt{ \frac{K\fdim\edim}{m} } \left( 2 M + \sqrt{ 2K \log\frac{2\fdim\edim}{\rho} } \right). 
  \end{equation*}
\end{proof}

We here prove Theorem \ref{expected_convergence_thm} by using statistical guarantees of empirical functional gradients.

\begin{proof}[ Proof of Theorem \ref{expected_convergence_thm} ]
  For notational simplicity, we set $m\leftarrow \floor[n/T]$ and $\delta \leftarrow \rho/T$.
  We first note that
  \begin{align*}
    \pd<\partial_\phi \regrisk(\phi_t,w_0),T_{k_t,m}\partial_\phi \regrisk_m(\phi_t,w_0)>_{L_2^\fdim(\tpr_{X})}
    &= \frac{1}{m}\sum_{j=1}^m \expec_{\tpr_X}[ \partial_\phi \regrisk(\phi_t,w_0)(X)^\top \partial_\phi \regrisk_m(\phi_t,w_0)(x_j) k_t(X,x_j) ] \\
    &= \frac{1}{m}\sum_{j=1}^m T_{k_t}\partial_\phi \regrisk(\phi_t,w_0)(x_j)^\top \partial_\phi \regrisk_m(\phi_t,w_0)(x_j) \\
    &= \pd< T_{k_t}\partial_\phi \regrisk(\phi_t,w_0), \partial_\phi \regrisk_m(\phi_t,w_0) >_{L_2^\fdim(\tpr_{m,X})}.
  \end{align*}

  Noting that $\| \partial_zl(\phi_t(x_j),y_j,w_0) \|_2 \leq \beta_{\|w_0\|_2}$ by Assumption \ref{smoothness_assumption_1},
  and applying Proposition \ref{stat_error_prop} for all $t \in \{0,\ldots,T-1\}$ independently, 
  it follows that with probability at least $1-T\delta$ (i.e., $1-\rho$) for $\forall t \in \{0,\ldots,T-1\}$
  \begin{align}
    &\hspace{-10mm} \left|\pd<\partial_\phi \regrisk(\phi_t,w_0),T_{k_t,m}\partial_\phi \regrisk_m(\phi_t,w_0)>_{L_2^\fdim(\tpr_{X})}
    -  \pd< T_{k_t,m}\partial_\phi \regrisk_m(\phi_t,w_0), \partial_\phi \regrisk_m(\phi_t,w_0) >_{L_2^\fdim(\tpr_{m,X})} \right|\notag\\
    &\leq \| T_{k_t}\partial_\phi \regrisk(\phi_t,w_0) - T_{k_t,m}\partial_\phi \regrisk_m(\phi_t,w_0) \|_{L_2^\fdim(\tpr_{m,X})} 
      \| \partial_\phi \regrisk_m(\phi_t,w_0) \|_{L_2^\fdim(\tpr_{m,X})} \notag\\
    &\leq \beta_{\|w_0\|_2}\epsilon(m,\delta). \label{stat_error_bound_2}
  \end{align}

  We next give the following bound.
  \begin{equation}
    \| T_{k_t}\partial_\phi \regrisk_m(\phi_t,w_0) \|_{L_2^\fdim(\tpr_X)}^2
    = \expec_{\tpr_X}\left\| \frac{1}{m}\sum_{j=1}^m \partial_zl(\phi_t(x_i),y_i,w_0)k_t(x_i,X) \right\|_2^2
    \leq \beta_{\|w_0\|_2}^2 K^2. \label{expected_convergence_thm_ineq_1}
  \end{equation}

 On the other hand, we get by Proposition \ref{taylor_prop}
  \begin{equation}
    \regrisk(\phi_{t+1},w_0)
    \leq \regrisk(\phi_{t+1},w_0)
    - \eta \pd<\partial_\phi \regrisk(\phi_t,w_0),T_{k_t,m}\partial_\phi \regrisk_m(\phi_t,w_0)>_{L_2^\fdim(\tpr_{X})}
    + \frac{\eta^2 A_{\|w_0\|_2}}{2}  \| T_{k_t}\partial_\phi \regrisk_m(\phi_t,w_0) \|_{L_2^\fdim(\tpr_X)}^2. \label{expected_taylor_prop}
  \end{equation}    

  Combining inequalities (\ref{stat_error_bound_2}), (\ref{expected_convergence_thm_ineq_1}), and (\ref{expected_taylor_prop}),
  we have with probability at least $1-T\delta$ for $t \in \{0,\ldots,T-1 \}$,
  \begin{equation*}
    \regrisk(\phi_{t+1},w_0)
    \leq \regrisk(\phi_{t+1},w_0)
    - \eta \| T_{k_t,m}\partial_\phi \regrisk_m(\phi_t,w_0) \|_{k_t}^2
    + \eta \beta_{\|w_0\|_2}\epsilon(m,\delta)
    +  \frac{\eta^2 \beta_{\|w_0\|_2}^2 K^2 A_{\|w_0\|_2}}{2}.
  \end{equation*}

  By Summing this inequality over $t\in \{0,\ldots,T-1\}$ and dividing by $T$, we get with probability $1-T\delta$
  \begin{equation*}
    \frac{1}{T}\sum_{t=0}^{T-1} \| T_{k_t,m}\partial_\phi \regrisk_m(\phi_t,w_0) \|_{k_t}^2
    \leq \frac{\regrisk(\phi_0,w_0)}{\eta T} + \beta_{\|w_0\|_2}\epsilon(m,\delta) + \frac{\eta \beta_{\|w_0\|_2}^2 K^2 A_{\|w_0\|_2}}{2}.
  \end{equation*}

  Thus by Assumption \ref{kernel_choice_assumption} and the assumption on $w_0^\top w_0$, we get
  \begin{equation}
    \frac{1}{T}\sum_{t=0}^{T-1} \| \nabla_f \risk_m(w_0^\top\phi_t) \|_{L_p^\fdim(\tpr_{m,X})}^q
    \leq \frac{1}{\gamma \sigma^q} \left\{
      \frac{\regrisk(\phi_0,w_0)}{\eta T} + \beta_{\|w_0\|_2}\epsilon(m,\delta) + \frac{\eta \beta_{\|w_0\|_2}^2 K^2 A_{\|w_0\|_2}}{2} + \gamma \epsilon \right\}.
    \label{exp_conv_thm_ineq_1}
  \end{equation}

  To clarify the relationship between $\|\nabla_f \risk_m(f)\|_{L_1^\nclass(\tpr_{m,X})}$ and $\|\nabla_f \risk(f)\|_{L_1^\nclass(\tpr_{X})}$,
  we take an expectation of the former term with respect to samples $(X_j,Y_j)_{j=1}^m \sim \tpr^m$.
  Since $\| \partial_\zeta l(\zeta,y) \|_2 \leq B$, we obtain

  \begin{align*}
    \expec_{(X_j,Y_j)_{j=1}^m \sim \tpr^m} \|\nabla_f \risk_m(f)\|_{L_1^\nclass(\tpr_{m,X})}
    &= \expec_{(X,Y)\sim \tpr_m} \| \partial_\zeta l(f(X),Y) \|_2\\
    &\geq \frac{1}{B}\expec_{(X,Y)\sim \tpr_m} \| \partial_\zeta l(f(X),Y) \|_2^2\\
    &\geq \frac{1}{B} \expec_{\tpr_{m,X}} \| \expec_{\tpr(Y|X)} [\partial_\zeta l(f(X),Y)] \|_2^2\\
    &= \frac{1}{B} \expec_{\tpr_{m,X}} \| \nabla_f \risk(f)(X) \|_2^2\\
    &= \frac{1}{B} \| \nabla_f \risk(f) \|_{L_2^\nclass(\tpr_X)}^2.
  \end{align*}

  Hence, applying Hoeffding's inequality with $\delta \leftarrow \rho/T$ to $\expec_{(X_j,Y_j)_{j=1}^m \sim \tpr^m} \|\nabla_f \risk_m(\predict{w_0}{\phi_t})\|_{L_1^\nclass(\tpr_{m,X})}$
  for all $t\in \{0,\ldots,T-1\}$ independently, we find that with probability $1-T\delta$ for $\forall t\in \{0,\ldots,T-1\}$,
  \begin{equation}
    \|\nabla_f \risk_m(\predict{w_0}{\phi_t})\|_{L_1^\nclass(\tpr_{m,X})} + B\sqrt{\frac{2}{m} \log\frac{1}{\delta}}
    \geq \expec_{\sim \tpr^m} \|\nabla_f \risk_m(\predict{w_0}{\phi_t})\|_{L_1^\nclass(\tpr_{m,X})}
    \geq \frac{1}{B} \| \nabla_f \risk(\predict{w_0}{\phi_t}) \|_{L_1^\nclass(\tpr_X)}^2, \label{funct_grad_concentration}
  \end{equation}
  where we used for the last inequality
  $\|\cdot \|_{L_2^\nclass(\tpr_X)}^2 \geq \|\cdot \|_{L_1^\nclass(\tpr_X)}^2$.

  We set $t_* = \argmin_{t \in \{0,\ldots,T-1\}} \| \nabla_f \risk_m(\predict{w_0}{\phi_t}) \|_{L_p^\fdim(\tpr_{m,X})}$.
  Combining inequalities (\ref{exp_conv_thm_ineq_1}) and (\ref{funct_grad_concentration}) and noting $p\geq 1$, we get with probability at least $1-2T\delta$,
  \begin{equation*}
    \frac{1}{B} \| \nabla_f \risk(\predict{w_0}{\phi_{t_*}}) \|_{L_1^\nclass(\tpr_X)}^2
    \leq B\sqrt{\frac{2}{m} \log\frac{1}{\delta}}
    + \frac{1}{\gamma^{1/q} \sigma} \left\{
      \frac{\regrisk(\phi_0,w_0)}{\eta T} + \beta_{\|w_0\|_2}\epsilon(m,\delta) + \frac{\eta \beta_{\|w_0\|_2}^2 K^2 A_{\|w_0\|_2}}{2} + \gamma \epsilon \right\}^{\frac{1}{q}}.
  \end{equation*}
  Noting that $\sqrt{a+b} \leq \sqrt{a} + \sqrt{b}$ for $a,b>0$, we finally obtain
  \begin{equation*}
    \| \nabla_f \risk(\predict{w_0}{\phi_{t_*}}) \|_{L_1^\nclass(\tpr_X)}
    \leq B \left(\frac{2}{m} \log\frac{1}{\delta} \right)^{\frac{1}{4}}
    + \sqrt{ \frac{B}{\gamma^{1/q} \sigma} } \left\{
      \frac{\regrisk(\phi_0,w_0)}{\eta T} + \beta_{\|w_0\|_2}\epsilon(m,\delta) + \frac{\eta \beta_{\|w_0\|_2}^2 K^2 A_{\|w_0\|_2}}{2} + \gamma \epsilon \right\}^{\frac{1}{2q}}.
  \end{equation*}
  Recalling that $m\leftarrow \floor[n/T]$ and $\delta \leftarrow \rho/T$, the proof is finished.
\end{proof}



\end{document}